%% file: paper_nips.tex
\documentclass{article}

\usepackage{fullpage}
\usepackage{authblk}

\usepackage[utf8]{inputenc} 
\usepackage[T1]{fontenc}    
\usepackage{hyperref}       
\usepackage{url}            
\usepackage{booktabs}       
\usepackage{amsfonts}       
\usepackage{nicefrac}       
\usepackage{microtype}      

\usepackage{amsfonts}
\usepackage{multirow}
\usepackage{amsmath}
\usepackage{amssymb}
\usepackage{tikz}
\usepackage{caption}
\usepackage{subcaption} 

\usepackage[utf8]{inputenc}
\usepackage{pgfplots}

\newcommand{\sign}{\mathrm{sign}}
\newcommand{\ReLU}{\mathrm{ReLU}}

\newcommand{\N}{\mathbb{N}}

\newcommand{\norm}[1]{\|#1\|}
\renewcommand{\H}{\mathcal{H}}

\newcommand{\reals}{\mathbb{R}}

\newcommand{\erf}{\mathrm{erf}}

\newcommand{\inner}[1]{\langle #1 \rangle}

\newcommand{\poly}{\mathrm{poly}}

\newtheorem{theorem}{Theorem}

\newtheorem{lemma}{Lemma}
\newtheorem{corollary}{Corollary}
\newcommand{\BlackBox}{\rule{1.5ex}{1.5ex}}  
\newenvironment{proof}{\par\noindent{\bf Proof\ }}{\hfill\BlackBox\\[2mm]}

\newcommand{\lemref}[1]{Lemma~\ref{#1}}
\newcommand{\thmref}[1]{Theorem~\ref{#1}}

\newcommand{\figref}[1]{Figure~\ref{#1}}
\newcommand{\secref}[1]{Section~\ref{#1}}
\newcommand{\appref}[1]{Appendix~\ref{#1}}

\newcommand{\bx}{\mathbf{x}}
\newcommand{\bv}{\mathbf{v}}
\newcommand{\bu}{\mathbf{u}}
\newcommand{\bw}{\mathbf{w}}
\newcommand{\bz}{\mathbf{z}}
\newcommand{\bg}{\mathbf{g}}
\newcommand{\ba}{\mathbf{a}}
\newcommand{\bb}{\mathbf{b}}

\newcommand{\E}{\mathbb{E}}

\newcommand{\figsdir}{figs}  

\newlength\figureheight
\newlength\figurewidth
\title{Weight Sharing is Crucial to Succesful Optimization}

\author[1]{Shai Shalev-Shwartz}
\author[2]{Ohad Shamir}
\author[1]{Shaked Shammah}
\affil[1]{School of Computer Science and Engineering, The Hebrew University}
\affil[2]{Weizmann Institute of Science}

\begin{document}

\maketitle

\begin{abstract}
  Exploiting the great expressive power of Deep Neural Network
  architectures, relies on the ability to train them. While current
  theoretical work provides, mostly, results showing the hardness of
  this task, empirical evidence usually differs from this line, with
  success stories in abundance. A strong position among empirically
  successful architectures is captured by networks where extensive
  weight sharing is used, either by Convolutional or Recurrent
  layers. Additionally, characterizing specific aspects of different
  tasks, making them ``harder'' or ``easier'', is an interesting
  direction explored both theoretically and empirically. We consider a
  family of ConvNet architectures, and prove that weight sharing can
  be crucial, from an optimization point of view. We explore different
  notions of the frequency, of the target function, proving necessity
  of the target function having some low frequency components. This
  necessity is not sufficient - only with weight sharing can it be
  exploited, thus theoretically separating architectures using it,
  from others which do not. Our theoretical results are aligned with
  empirical experiments in an even more general setting, suggesting
  viability of examination of the role played by interleaving those
  aspects in broader families of tasks.
\end{abstract}

\section{Introduction}
There are many directions from which one can examine Deep Learning
(DL). Very popular is the direction of empirical success, where
extensive research effort had resulted in state-of-the-art,
overwhelming breakthroughs, in a wide range of tasks. One may need to
read between the lines to gain insights regarding the difficulties
which faced the practitioners on their way to success. This is true, in
particular, when regarding the optimization process. While sample
complexity issues are usually straightforward to deal with (``add more
data''), and expressive power of the used networks is generally more
than sufficient, successful optimization, and in particular,
success of Gradient Descent (GD), is left as a mystery. What
aspects of a task cause the general gradient-based DL approach to
succeed or fail? 

In this paper, we study this question for a simple, yet powerful,
ConvNet architecture: one convolutional layer, mapping $k$ image
patches, each of dimension $d$, into $k$ scalars, followed by a non
linear activation, a fully connected (FC) layer with $\ReLU$ activation,
and a final FC layer with one output neuron. Most if not all DL
practitioners would have known this ``recipe'' by heart. We think of
$k$ as relatively smaller than $d$: for example, $d=75$ and $k=10$,
corresponding to a $5\times5\times3$ convolution kernel over a small
color image. This family of architectures, as trivial and simplistic
as it is, can provide us with very fertile ground on which to examine
interesting empirical phenomena. We assume that the target function
which we are trying to learn is generated by a network of the exact
same architecture, and learning is performed with a very large
training set. Therefore, there are neither expressiveness nor
overfitting issues, which enables us to focus solely on the success of
GD.

Any target function generated by the above architecture, can be
thought of as a composition of two functions: the convolutional first
layer (with its non linearity), denoted $h^*:\reals^{dk}\to\reals^k$,
subsequently fed into the second part of the network, denoted
$g^*:\reals^{k}\to\reals$. We underscore two properties of DL tasks
that control GD's success or failure. The first property, uses notions
of frequency, from Fourier analysis, to characterize ``hardness'' of a
task. The second property, distinguishes between Convolutional layers,
in which weights are shared, and Fully Connected (FC) ones.

Since the target function is the composition $g^*\circ h^*$, it is
natural to start by understanding the success of GD when one of the
target function's components is fixed and known, with only the other
being learnt. For the case of known $h^*$, because $k$ is small, it is
possible to show that under some mild conditions, the problem of
learning $g^*$ is not hard (see \appref{app:easy_g}). A more
interesting case is when $g^*$ is known, and our task is to learn
$h^*$. In \cite{shalev2017failures}, it has been shown that no
Gradient Based algorithm can succeed in learning $h^*$ if $g^*$ is the
parity of the signs of its input. The parity function consists of the
highest frequency of the Fourier expansion for functions over the
boolean cube.  In this paper, we prove that $h^*$ can be learnt by GD,
if the Fourier expansion of $g^*$ contains both a frequency $1$
element and a higher frequency element, namely, a combination both
high and low frequencies. We further prove, that this positive result
depends on our architecture for learning $h^*$: if the convolutional
layer is replaced by a FC one, then GD will fail. It is the
combination of $g^*$ having a low frequency component, along with the
weight sharing in our architecture, that is essential for
success. Formal statements of these claims are given in
\secref{sec:sum_parities}.

Naturally, mathematically analyzing the convergence properties of GD
in this highly non convex problem, relies on some simplifying
assumptions. A major one, is the assumption that $g^*$ is known. We
start the paper, in \secref{sec:experiment}, by empirically
demonstrating that our theoretical results seem to hold even in the
case of learning simultaneously both $h^*$ and $g^*$. In
\secref{sec:sum_parities} we prove our main result, but before that,
we highlight, in \secref{sec:sum_cos}, the same phenomena, albeit in a
simpler setting, where
$g^*(\bz)=c_kz_1+\cos(\sum_{i=1}^k z_i)$. Although somewhat
synthetic, this setting does maintain the flavour of separation
between low and high frequencies, this time from the perspective of
Fourier analysis over $\reals^k$. Additionally, it allows for a
relatively simple, direct proof technique, showing a computational
separation between learning with or without weight sharing, where an
exponential gap in time complexity of GD is proven to exist.

\subsection{Related Work}
Recently, several works have attempted to study the optimization
performance of gradient-based methods for neural networks. To mention
just a few pertinent examples,
\cite{safran2016quality,choromanska2015loss,soudry2016no,haeffele2015global,hardt2016identity}
consider the optimization landscape for various networks, showing it
has favorable properties under various assumptions, but does not
consider the behavior of a specific algorithm. Other works, such as
\cite{livni2014computational,arora2014provable,janzamin2015beating,zhang2016l1},
show how certain neural networks can be learned under (generally
strong) assumptions, but not with standard gradient-based
methods. More closer to our work,
\cite{andoni2014learning,brutzkus2017globally,daniely2017sgd} provide
positive learning results using gradient-based algorithms, but do not
show the benefit of a convolutional architecture for optimization
performance, compared to a fully-connected architecture.  The hardness
of learning in the case of Boolean functions, using the degree of the
target function, was discussed in the statistical queries literature,
for instance in \cite{dachman2015approximate}.  In terms of
techniques, our construction is inspired by target functions proposed
in \cite{shalev2017failures,shamir2016distribution}, and based on
ideas from the statistical queries literature
(e.g. \cite{blum1994weakly}), to study the difficulty of learning with
gradient-based methods.

\section{Empirical Demonstration}\label{sec:experiment}

The target function we wish to learn is of the form $g^*(h^*(\bx))$, 
where $\bx = (\bx_1,\ldots,\bx_k)$, with $\bx_i \in \reals^d$ for
every $i$. The function $h^*$ is parameterized by a vector $\bu_0 \in
\reals^d$ and is defined as $h^*(\bx) =
(\sigma(\bu_0^\top,\bx_1),\ldots,\sigma(\bu_0^\top \bx_k))$, where we
chose $\sigma$ to be the $\tanh$ function, as a smooth approximation
of the $\sign$ function. We can therefore think of the input to $g^*$
as approximately being from $\{\pm 1\}^k$.  In our experiment we vary four parameters:
\begin{itemize}
\item The value of $g^*$ is set to be
either $g^*_{\mathrm{low}}(\bz) := \bz_1$, or 
$g^*_{\mathrm{high}}(\bz) = \prod_{i=1}^5 \bz_i$, or
$g^*_{\mathrm{both}}(\bz) = g^*_{\mathrm{low}}(\bz) +
g^*_{\mathrm{high}}(\bz)$. 
\item Weight Sharing (WS) vs. Fully Connected (FC): we also learn a
  compositional function $g(h(\bx))$, and the function $h(\bx)$ can be
  either with weight sharing, $h(\bx) =
  (\sigma(\bw_0^\top,\bx_1),\ldots,\sigma(\bw_0^\top \bx_k))$, where
  we learn the vector $\bw_0 \in \reals^d$, or with fully connected
  architecture, namely,  $h(\bx) =
  (\sigma(\bw_1^\top,\bx_1),\ldots,\sigma(\bw_k^\top \bx_k))$, where
  we learn the vector $\bw = (\bw_1,\ldots,\bw_k)$. 
\item Known vs. Unknown $g^*$: for the function $g$ we either use $g =
  g^*$ or learn $g$ as well by using the following architecture: FC
  layer with $50$ outputs, ReLU, and FC layer with a single output.
\item Input Distribution: we either sample $\bx$ from a Gaussian
  distribution, or use real image patches of size $10\times10$ from
  the MNIST data set, normalized to have zero mean.
\end{itemize}

We train all of our networks with SGD, with $\eta=0.5$, batch size
$128$, and the Squared Loss, for $3000$ iterations. The vectors
$\bx_i$ were generated by sampling from a normal distribution. The
results of these experiments are depicted on the $6$ graphs of 
\figref{fig:experiments}. 

The graphs reveals several interesting observations. The first is the
clear failure, of both WS and FC architectures, for both real and
Gaussian data, and for both known and unknown $g^*$, when the target
function is $g^*_{\mathrm{high}}$, and the contrasting success when it
is $g^*_{\mathrm{low}}$. To explain this difference, let us
characterize the $g^*$s using tools from Fourier analysis of real
functions over the boolean cube. The representation of such functions
in the Fourier basis can be used to define many different meaningful
characterizations. Perhaps one of the most natural ones is the degree,
or frequency of the function. Specifically, in our case,
$g^*_{\mathrm{low}}$ is a basis function of degree $1$, while
$g^*_{\mathrm{high}}$ is a basis function of degree $5$. Our
theoretical analysis shows that the number of GD iterations required
to learn $h^*$ when $g^*$ is a basis function grows as
$d^{\textrm{degree}}$. In our experiment $d = 75$, or $100$ for the
MNIST patches, and we observe a clear separation already between
degree $1$ and degree $5$.

Next, since real-world functions will likely contain several frequencies,
it is natural to study functions that combine many basis
elements. As a first step, we turn to observe the performance for
$g^*_{\mathrm{both}}$. Here, we suddenly see a strong separation
between the WS and FC architectures: the optimization converges very
quickly for the WS architecture while for the FC one, the high
frequency component has not been learnt. Our theoretical analysis
proves that, indeed, the number of GD iterations required by the FC
architecture still grows as $d^{\textrm{high-degree}}$, while for the
WS architecture, the required number of iterations is only polynomial
in the high degree.\footnote{We emphasize that this exponential gap between 
WS and FC is due to \emph{computational} reasons and not due to
overfitting. The difference in sample complexity between the two architectures is
only a factor of $k$, and in both cases the training set size is
sufficiently large.}
 Intuitively, the low frequency term directs the
single, shared, weight vector towards the optimum. Once $h$ converged
to $h^*$, even if $g^*$ is unknown, the GD process succeeds in
learning it, because $k$ is small. In contrast, without weight
sharing, the components of $\bw$ that appear only in the high degree
term are not being learnt, as was the case for $g^*_{\mathrm{high}}$.

Finally, while our analysis proves the positive and negative results
for the case of known $g^*$ where $\bx$ is normally distributed, the
graphs show that even in the more general case, when $g^*$ is also
being learnt, and even if the data is natural, the picture remains
roughly the same.\footnote{The only difference across this aspect,
  when learning a degree $1$ parity with a convolutional architecture,
  between known and unknown $g^*$, for Gaussian data,
  is perhaps due to smaller Signal to Noise Ratio in the case of
  learning $g^*$, as suggested in \cite{shalev2017failures}.}

\setlength\figureheight{4cm}
\setlength\figurewidth{4cm}
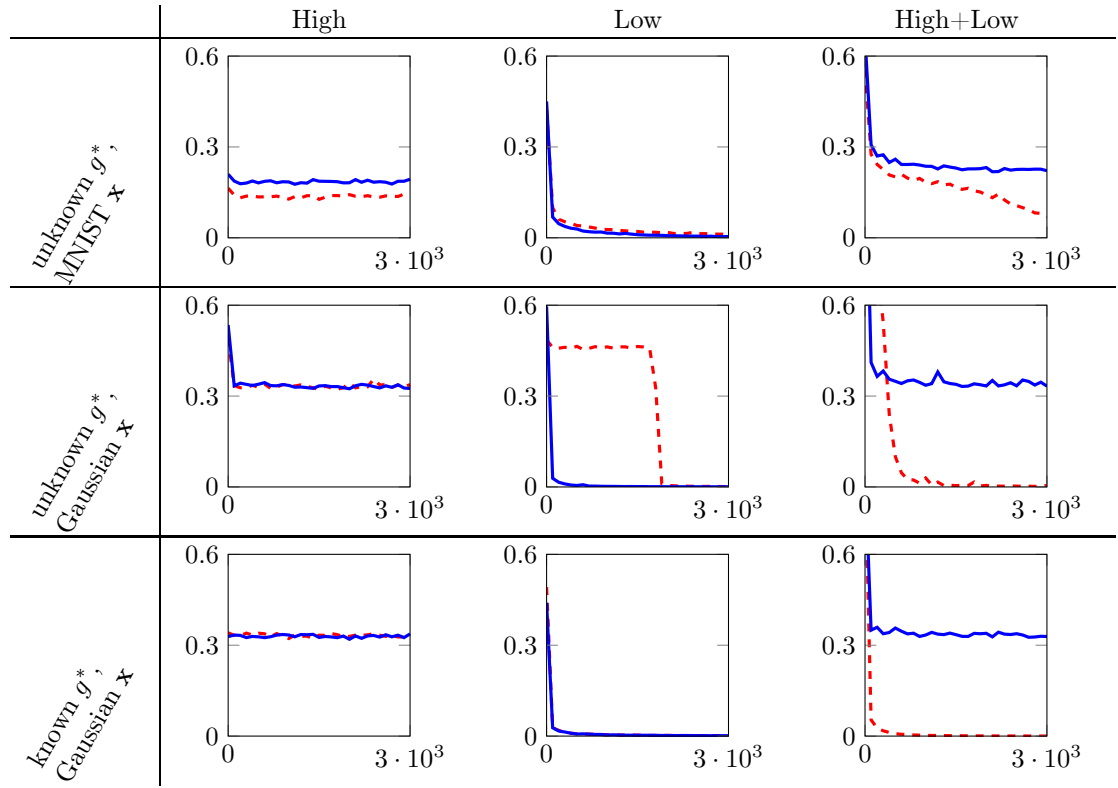
\begin{figure}
\begin{tabular}{l|ccc}
 & High & Low & High+Low \\ \hline
\rotatebox{60}{\parbox{2cm}{unknown $g^*$,\\ MNIST $\bx$}} & \input{\figsdir/unknown_g_5_mnist_mytikz.tex} 
        & \input{\figsdir/unknown_g_1_mnist_mytikz.tex} 
        & \input{\figsdir/unknown_g_1_5_mnist_mytikz.tex} 
  \\ \hline
\rotatebox{60}{\parbox{2cm}{unknown $g^*$,\\ Gaussian $\bx$}} & \input{\figsdir/unknown_g_5_mytikz.tex} 
        & \input{\figsdir/unknown_g_1_mytikz.tex} 
        & \input{\figsdir/unknown_g_1_5_mytikz.tex} 
  \\ \hline
\rotatebox{60}{\parbox{2cm}{known $g^*$,\\ Gaussian $\bx$}} & \input{\figsdir/fix_g_5_mytikz.tex} 
        & \input{\figsdir/fix_g_1_mytikz.tex} 
        & \input{\figsdir/fix_g_1_5_mytikz.tex} 
\end{tabular}
\caption{Loss as a function of SGD iterations for the Convolutional 
  (dashed-red) and FC (solid-blue) architectures. Columns correspond
  to frequencies of $g^*$ and rows correspond to whether $g^*$ is
  known or also being learnt. Our theory proves the bottom row. As can
be seen, the top row behaves similarly.}  \label{fig:experiments}
\end{figure}

\section{Sum of Low and High Degree Waves}\label{sec:sum_cos}

In this section we provide our first separation result between the WS and FC
architectures. 

Let $\bx=(\bx_1,\ldots,\bx_k)\in \reals^{dk}, \bw=(\bw_1,\ldots,\bw_k)\in
\reals^{dk}$ denote input elements, and weight vectors,
respectively. Define:
\[
p_{\bw}(\bx) = 
c_k \bw_1^\top\bx_1+\cos\left(\sum_{i=1}^{k}\bw_i^\top\bx_i\right),
\]
where $c_k$ is any parameter $\geq 3\sqrt{k}$. As in the introduction,
we define a sub-family of functions, parameterized by $\bu_0$, and
defined as:
\[
p_{\bu_0}(\bx) = 
c_k \bu_0^\top\bx_1+\cos\left(\sum_{i=1}^{k}\bu_0^\top\bx_i\right).
\]
For simplicity of notation, when using the 0 subscript for the weight
vector, we refer to an element of the WS sub-family. Additionally, we
use $\bar{\bu}_0$ to denote the vector composed of $k$ duplicates of
$\bu_0$, namely $(\bu_0,\bu_0,\ldots,\bu_0)$.  Consider the
objective
\[
F(\bw) = 
\E_{\bx}\left[\frac{1}{2}\left(p_{\bw}(\bx)-p_{\bu_0}(\bx)\right)^2\right],
\]
where $\bx$ is standard Gaussian. We consider the gap between
optimizing a FC architecture, namely, one parameterized by $\bw$, and
a WS one, parameterized by a single weight vector $\bw_0$. We note
that our choice of $c_k$ is merely to simplify the proofs --
convergence guarantees can be proven for other choices of $c_k$
(including $c_k=1$), but the proof requires more effort.

\subsection{Hardness Result for Optimizing $F$ using GD - FC
  Architecture}\label{sec:hard_FC_cos}

\begin{theorem}\label{thm:cos_fc}
  Assuming $k>1$, the following holds for some numerical constants
  $c_1,c_2,c_3,c_4$: For any $\bw$ such that
  $\norm{(\bw_2,\ldots,\bw_k)} \in
  \left[\frac{\sqrt{k-1}}{3}\cdot\norm{{\bu}_0},\frac{\sqrt{k-1}}{2}\cdot\norm{\bu_0}\right]$,
  it holds that
\[
\left\|\frac{\partial}{\partial (\bw_2,\ldots,\bw_k)}F(\bw)\right\|~\leq ~
c_1\sqrt{k}\norm{\bu_0}\exp\left(-c_2 k\norm{\bu_0}^2\right)~.
\]
Moreover, for any $\bw$ such that $\norm{(\bw_2,\ldots,\bw_k)} \leq 
\frac{\sqrt{k-1}}{2}\norm{\bu_0}$, it holds that
\[
F(\bw)-F(\bar{\bu}_0)~\geq~ 1-c_3\exp(-c_4k\norm{\bu_0}^2).
\]
\end{theorem}
The proof is given in \appref{app:sec:hard_FC_cos}.  To understand the
implication of the theorem, consider a gradient-based method, starting
at some initial point $\bw^{(0)}$ such that
$\norm{(\bw^{(0)}_2,\ldots,\bw^{(0)}_k)} \leq
\frac{\sqrt{k-1}}{3}\cdot\norm{{\bu}_0}$ (a reasonable assumption).
In that case, the theorem implies that the algorithm will need to
cross the ring
\[
\left\{(\bw_2,\ldots,\bw_k):\norm{(\bw_2,\ldots,\bw_k)} \in 
\left[\frac{\sqrt{k-1}}{3}\cdot\norm{{\bu}_0},\frac{\sqrt{k-1}}{2}\cdot\norm{\bu_0}\right]
\right\}
\]
w.r.t. $(\bw_2,\ldots,\bw_k)$, to get to a solution which ensures sub-constant 
error. However, in that ring, the gradients are essentially exponentially small 
in $k\norm{\bu_0}^2$. This implies that performing gradient steps with any 
bounded step size, one would need exponentially 
many iterations (in $k\norm{\bu_0}^2$) to achieve sub-constant
error. 

\subsection{Positive Result for Optimizing $F$ using GD - WS
  Architecture}\label{sec:easy_WS_cos}
 We show that when using a WS architecture, the objective is
 transformed to be strongly convex, making for simple proof techniques
 being applicable. The proof is given in \appref{app:sec:easy_WS_cos}.

\begin{theorem}\label{thm:cos_conv}
Using the WS architecture, $F$ is strongly convex, minimized at $\bu_0$, and satisfies 
$\frac{\max_{\bw_0}\lambda_{\max}(\nabla^2 
F(\bw_0))}{\min_{\bw_0}\lambda_{\min}(\nabla^2 F(\bw_0))}\leq 5$, where 
$\lambda_{\max}(M)$ and $\lambda_{\min}(M)$ are the top and bottom eigenvalues 
of a positive definite matrix $M$. Hence, gradient descent starting from any 
point $\bw^{(0)}$, and with an appropriate 
step size, will reach a point $\bw$ satisfying 
$\norm{\bw-\bu_0}\leq \epsilon$ in at most
$5\cdot\log(\norm{\bw_{init}-\bu_0}/\epsilon)$ iterations. 
\end{theorem}

\section{Sum of Low and High Degree Parities}\label{sec:sum_parities}

This section formalizes our main result, namely, a separation between
WS and FC architectures for learning a target function, $g^* \circ
h^*$, when $g^* = g^*_{\mathrm{both}}$ is comprised of both high and low
frequencies. The negative result for the FC architecture is given in
\thmref{thm:main_FC_fail} and the positive result for the WS architecture is given in
\thmref{thm:positive_WS}. 

\subsection{Definitions, Notation}
Let $\bx$ denote a k-tuple $(\bx_1,\ldots,\bx_k)$ of input instances,
and assume that each $\bx_l$ is i.i.d. standard Gaussian in
$\reals^d$. Let $\sigma : \reals \to [-1,1]$ be some smooth approximation
of the sign function. To simplify the analysis, we use the $\erf$
function:
$
\erf(x)=\frac{1}{\sqrt{\pi}}\int_{-x}^x e^{-t^2}\mathrm{d}t
$. 
We believe that our analysis holds for additional functions, such as
the popular $\tanh$ function.  Define, for $k\in\N$, a family of
functions $\H_{FC}^{(k)}$, parameterized by $\bw\in (\reals^d)^k$, and
defined by:
\[
p^{(k)}_{\bw}(\bx) = \prod_{l=1}^{k}\sigma(\bw_l^\top \bx_l).
\]
Let us define a subclass, parameterized by $\bw_0\in \reals^{d}$ and
denoted $\H^{(k)}_{WS}$, by:
\[
p^{(k)}_{\bw_0}(\bx) = \prod_{l=1}^{k}\sigma(\bw_0^\top \bx_l).
\]
Note that the difference between $\H_{WS}^{(k)}$ and $\H^{(k)}_{FC}$
is that elements in $\H_{WS}^{(k)}$ are satisfying the condition that
for all $l$, $\bw_l=\bw_0$ for some $\bw_0\in \reals^d$. For ease of
notation, we will refer to elements in $\H_{WS}^{(k)}$ and
$ \H^{(k)}_{FC}$ by $p^{(k)}_{\bw_0}$ and $p^{(k)}_{\bw}$,
respectively.

An extension of these families, denoted $\H_{WS}^{(1,k)}$ and
$\H^{(1,k)}_{FC}$, is defined, for $\bw_0\in \reals^d$ and
$\bw\in (\reals^{d})^k$ respectively, as the sum of the two
corresponding functions of the $\H^{(1)}, \H^{(k)}$ classes. Namely,
for the FC class,
\[
p^{(1,k)}_{\bw}(\bx) = \sigma(\bw_1^\top \bx_1) + \prod_{l=1}^{k}\sigma(\bw_l^\top \bx_l),
\]
with the definition for the WS class following as a special case.

Let the objective $F^{(k)}(\bw)$, w.r.t. some target function
$p^{(k)}_{\bu_0}$ be the expected squared loss,
\[
F^{(k)}(\bw) = \E_{\bx}\left[\frac{1}{2}(p^{(k)}_{\bw}(\bx)-p^{(k)}_{\bu_0}(\bx))^2\right],
\]
with a similar definition for $F^{(1,k)}(\bw)$, namely 
\begin{equation} \label{eqn:F1kdef}
 F^{(1,k)}(\bw) =
\E_{\bx}\left[\frac{1}{2}(p^{(1,k)}_{\bw}(\bx)-p^{(1,k)}_{\bu_0}(\bx))^2\right] ~.
\end{equation}

The following definition and lemma, due to
\cite{williams1997computing}, will be useful in our analysis.
\begin{lemma}[\cite{williams1997computing}] \label{lem:williams}
Let $\sigma$ be the $\erf$ function. Then, 
For every pair of vectors $\bu,\bv\in\reals^d$ we have
\[ V_{\sigma}(\bu,\bv) ~:=~\E_{\bx \sim N(0,I)}\left[\sigma(\bw^\top\bx)
  \sigma(\bu^\top\bx)\right] ~=~ \frac{2}{\pi}\sin^{-1}\left(\frac{2\bu^\top
      \bv}{\sqrt{1+2\|\bu\|^2}\sqrt{1+2\|\bv\|^2}} \right) ~,
\]
where $N(0,I)$ is the standard Gaussian distribution.
\end{lemma}

\subsection{Non degeneracy depending on
  $\|\bu_0\|^2$}\label{sec:non_degenerate}
Note that as $|\sigma|<1$, for large values of $k$, and small
values of $\bu_0^\top\bx_l$, we have that $p_{\bu_0}(\bx)$ is
vanishing exponentially.  We show that in the case of large enough
$\|\bu_0\|^2$, depending on $k$, the target function's expected norm is
lower bounded, hence overcoming this possible degeneracy. 
\begin{lemma}\label{lem:non_degenerate}
  If $\|\bu_0\|^2 \ge \frac{12}{\pi^2} k^2$ then
\[
\E_\bx\left[(p_{\bu_0}^{(k)}(\bx))^2\right] > \frac{1}{4}.
\]
\end{lemma}
The proof is given in \appref{app:sec:non_degenerate}.

\subsection{Exact Gradient Expressions}\label{sec:exact_expressions}

In order to analyze the dynamics of the Gradient Descent (GD)
optimization process, we examine the exact gradient expressions. The
proofs to the lemmas are given at \appref{app:sec:exact_expressions}

Recall the definition of $F^{(1,k)}$ from \eqref{eqn:F1kdef}. We first
show that $F^{(1,k)}$ equals the sum of $F^{(1)}$ and $F^{(k)}$, due
to independence of these two terms.
\begin{lemma}\label{lem:Fsum_sumF}
For both architectures, 
\[
F^{(1,k)}(\bw) = F^{(1)}(\bw)+F^{(k)}(\bw).
\]
\end{lemma}

Based on \lemref{lem:Fsum_sumF}, we have that $\nabla F^{(1,k)}(\bw) =
\nabla F^{(1)}(\bw) + \nabla F^{(k)}(\bw)$, hence it suffices to find an explicit
expression for $\nabla F^{(k)}(\bw)$, for every $k$. We have,
\[
\nabla F^{(k)}(\bw) = \E_{\bx}\left[p^{(k)}_{\bw}(\bx)\bg^{(k)}_{\bw}(\bx) - p^{(k)}_{\bu_0}(\bx)\bg^{(k)}_{\bw}(\bx)\right],
\]
where $\bg^{(k)}_{\bw}(\bx) := \nabla_{\bw} p^{(k)}_{\bw}(\bx)$ is the
gradient of the predictor w.r.t. the weight vector $\bw$. We first
show the following symmetric property.
\begin{lemma}\label{lem:pg_pmgm}
  For all $k,\bw,\bw'$, and $\bx$,
\[
p^{(k)}_\bw(\bx) g^{(k)}_{\bw'}(\bx)=p^{(k)}_\bw(-\bx) g^{(k)}_{\bw'}(-\bx).
\]
\end{lemma}
 We can now proceed to compute that exact gradients.
\begin{lemma}\label{lem:exact_gradient}
  Let $\bg_l^{(k)}(\bx)$ be the gradient of the predictor
  w.r.t. $\bw_l$, the weights corresponding to the $l$th input
  element. Then
\[
  \E_{\bx}\left[p^{(k)}_{\bw}(\bx)\bg_l^{(k)}(\bx)\right]=
  \left(\prod_{j\neq l}V_{\sigma}(\bw_j, \bw_j)\right)\cdot  c_1(\bw_l)\bw_l,
\]
where $0<c_1(\bw_l)<1$, is independent of $k$, and:
\[
  \E_{\bx}\left[p^{(k)}_{\bu_0}(\bx)\bg_l^{(k)}(\bx)\right]=
  \left(\prod_{j\neq
     l} V_\sigma(\bu_0,\bw_j)\right)\cdot \tilde{\bb}_l
\]
for some vector $\tilde{\bb}_l\in\mathrm{span}\{\bw_l,\bu_0\}$, independent
of $k$.
\end{lemma}
\begin{corollary}
For the WS architecture, we have that:
\[
\E_{\bx}\left[p^{(k)}_{\bw_0}(\bx)\bg_0^{(k)}(\bx)\right]=k\cdot V_{\sigma}(\bw_0,\bw_0)^{k-1}\cdot  c_1(\bw_0)\bw_0,
\]
where $0<c_1(\bw_0)<1$, is independent of $k$, and:
\[
\E_{\bx}\left[p^{(k)}_{\bu_0}(\bx)\bg_0^{(k)}(\bx)\right]=k\cdot V_\sigma(\bu_0,\bw_0)^{k-1}\cdot \tilde{\bb}
\]
for some vector $\tilde{\bb}\in\mathrm{span}\{\bw_0,\bu_0\}$, independent of $k$.
\end{corollary}
The corollary follows immediately from the fact that the gradient
w.r.t. $\bw_0$, is the sum of gradients of
each ``duplicate'' of $\bw_0$, when considering $p_{\bw_0}$ as a
member of $\H^{(k)}_{FC}$. Note that, when $\bu_0^\top \bw_0>0$
(which happens w.p. $1/2$ over symmetric initialization, and as we later show,
this property is preserved during a run of GD), $V_\sigma(\bu_0,\bw_0)>0$.
Hence in such case, the coefficient of $\tilde{\bb}$ is positive.

\subsection{Hardness Result for Optimizing $F^{(1,k)}$ using GD - FC Architecture}\label{sec:hard_FC}
Equipped with the results of previous sections, we obtain a
computational hardness result for learning a target function
$p^{(1,k)}_{\bu_0}$ using GD with the FC architecture, showing that the progress after
any polynomial number of iterations, is exponentially small. Proofs
are given in \appref{app:sec:hard_FC}.

\begin{theorem} \label{thm:main_FC_fail} Consider a GD algorithm for
  optimizing $F^{(1,k)}$ for the FC architecture, that uses a learning
  rate rule such that for every $t$, $\eta_t \in (0,1]$. Suppose that
  every coordinate of $\bw_l$ is initialized i.i.d. uniformly from
  $\{\pm c\}$ for some constant $c$. Then, with probability of at
  least $1 - 2ke^{-d^{1/3}/6}$ over the random initialization, after
  $T = o(d^{k/4})$ iterations, we will have
  $F^{(1,k)}(\bw^{(T)}) \ge 1/8$.
\end{theorem}
The main idea of the proof is to show that progress in direction which
improves the angle between $\bw_j$ and $\bu_0$ is exponentially small
- unless $\bw_j$ gets close to the origin. In that case, it is
``stuck'' there with no ability to progress.  The proof relies on the
following lemmas. The first one shows that a ``bad'' initialization,
namely, one for which the initialized vectors are almost orthogonal to
$\bu_0$, happens with overwhelming probability.
\begin{lemma}\label{lem:chernoff_FC}
  Assume each $\bw_l$ is chosen by sampling uniformly from
  $\{\pm c\}^d$ for some $c$. Then w.p.
  $> 1-2 k e^{- 0.5\, d^{1/3}}$ over the initialization
  $\bw^{(0)}$, for all $l\in[k]$,
  $|\inner{\frac{\bw_l}{\|\bw_l\|},
    \frac{\bu_0}{\|\bu_0\|}}|<d^{-1/3}$.
\end{lemma}
Next, we directly upper bound the value of $|V_\erf(\bw_j,\bu_0)|$,
s.t. for cases when $\bw_j,\bu_0$ are almost orthogonal, or, when
$\|\bw_j\|$ is very small, $|V_\erf(\bw_j,\bu_0)|$ is small too.
\begin{lemma}\label{lem:small_Verf}
 Let $c \in [0,1]$ and assume
  $|\inner{\frac{\bw_j}{\|\bw_j\|}, \frac{\bu_0}{\|\bu_0\|}}| <
  c$, or
  $\|\bw_j\| \leq c/\sqrt{2}$. Then
\[
|V_\erf(\bw_j,\bu_0)|<\frac{2\,c}{\pi} < c
\]
\end{lemma}
Finally, we use the fact that the target function is non trivial, from
\lemref{lem:non_degenerate}, in order to show that in the case when
not all of the weight vectors have converged, we suffer high loss.
\begin{lemma}\label{lem:non_degenerate_large_loss}
  Assume that for some
  $j$, it holds that either
  $|\inner{\frac{\bw_j}{\|\bw_j\|}, \frac{\bu_0}{\|\bu_0\|}}| <
  \sin\frac{\pi}{32}$, or
  $\|\bw_j\| \leq \frac{1}{\sqrt{2}}\sin\frac{\pi}{32}$. Then
\[
F^{(1,k)}(\bw) > \frac{1}{8}.
\]
\end{lemma}

\subsection{Positive Result for Optimizing $F^{(1,k)}$ using GD - WS
  Architecture}\label{sec:easy_WS}

We now turn to state our positive result for the WS architecture. 
The outline of the proof is given in the subsections below, and
additional proofs of intermediate results are given at \appref{app:sec:easy_WS}.
\begin{theorem}\label{thm:positive_WS}
Running (projected) GD, with respect to the objective $F^{(1,k)}$ with the WS
architecture, and with a constant learning rate for $T = \poly(k,1/\epsilon)$
iterations, yields $\|\bw_0^{(T)}-\bu_0\| \le \epsilon$. 
\end{theorem}
To prove the theorem, we analyze the optimization process by
separating it into two phases. During the first, the $\bw_0$ converges
to the direction of $\bu_0$. Then, in a second phase, its norm
converges to that of $\bu_0$. The combination of the theorems proven
in the next sections directly imply \thmref{thm:positive_WS}.

\subsection{Phase 1 - Angle Convergence}\label{sec:phase1}
Assume that $\|\bu_0\|^2 = \frac{12}{\pi^2} \,k^2$, large enough for non
degeneracy of the target function, as shown in
\secref{sec:non_degenerate}. Moreover, we can assume w.l.o.g., that $\bu_0 =
\sqrt{\frac{12}{\pi^2}} \,k \, e_2$, where $e_2 =
(0,1,0,\ldots,0)$. We can further assume w.l.o.g. that
$\mathrm{span}\{\bu_0,\bw_0\}=\mathrm{span}\{e_1,e_2\}$. By the random
initialization, it holds that $\bw_0^\top \bu_0 > 0$ with probability
$1/2$. We will assume that this is indeed the case. In addition, we
will assume, w.l.o.g., that the
first two coordinates of $\bw_0$ are non-negative. Finally, assume that
$\|\bw^{(t)}_0\| \le \|\bu_0\|$ for every $t$ (if this is not the
case, it is standard to add a projection onto this ball).
\begin{theorem}\label{thm:phase1}
  Let $\alpha^{(t)}$ be the angle between $\bu_0,\bw_0^{(t)}$. Then
  $\alpha^{(t+1)}\leq\alpha^{(t)}$. Moreover, for every $\epsilon > 0$, there exist
  $T=O((k/\epsilon)^3)$ s.t. $\alpha^{(T)}<\epsilon$.
\end{theorem}

\subsection{Phase 2}
We use the same assumptions as in \secref{sec:phase1}. We start off
with a theorem showing that the gradient of $F^{(1)}$ directs the
weights towards the optimum, $\bu_0$. The proof uses monotonicity of
$\sigma$, with similar techniques as found in
\cite{kalai2009isotron,kakade2011efficient,mei2016landscape}.
\begin{theorem}\label{thm:isotron_gradient_correlation}
For some $L^2(\tilde{s})=\Theta(1)$, 
\[
\inner{\bw_0-\bu_0,\nabla F^{(1)}(\bw_0)}\geq
  \frac{L^2(\tilde{s})}{k}\|\bw_0-\bu_0\|^2.
\]
\end{theorem}

After establishing the above result, we next show that when the angle
$\alpha$ between $\bw_0$ and $\bu_0$ is small (which we've shown is
the case in polynomial time, after the first phase of optimization, in
\secref{sec:phase1}), the gradient of $F^{(1,k)}$ has the same
property as the gradient of $F^{(1)}$, namely, it too points in a good
direction. The intuition is that when $\bw_0$ is close, in terms of
angle, to $\bu_0$, the gradient of $F^{(k)}$ becomes more similar to
the gradient of $F^{(1)}$, making it helpful too.
\begin{theorem}\label{thm:Fsum_gradient_correlation}
  Assume
  $\alpha^{(t)}<\min\{\tan^{-1}\left(\frac{\epsilon}{2\|\bu_0\|}\right),\sqrt{\frac{\epsilon}{\|\bu_0\|}}\}$
  and $\|\bw_0-\bu_0\| > \epsilon$. Then
\[
    \inner{\bw_0-\bu_0,\nabla F^{(1,k)}(\bw_0)}\geq
    \frac{L^2(\tilde{s})}{k}\|\bw_0-\bu_0\|^2,
\] 
for $L^2(\tilde{s})=\Theta(1)$, as in \thmref{thm:isotron_gradient_correlation}.
\end{theorem}
We now use the above lower bound over the inner product between the
gradient and the optimal optimization step, to show that for any $\epsilon>0$, after a
polynomial number of iterations of GD, we converge to a solution $\bw_0$
for which $\|\bw_0-\bu_0\|<\epsilon$. 
\begin{theorem}\label{thm:phase2}
  Assume $\bw^{(t')}$ is such that
  $\alpha^{(t')}<\min\{\tan^{-1}\left(\frac{\epsilon}{2\|\bu_0\|}\right),\sqrt{\frac{\epsilon}{\|\bu_0\|}}\}$. Then,
  after at most $\poly(k,1/\epsilon)$ additional iterations we must
  have that 
\[
  \|\bw_0^{(t)}-\bu_0\| \leq \epsilon ~.
\]
\end{theorem}

\paragraph{Acknowledgements:} This research is supported in part by
the Intel Collaborative Research Institute for Computational
Intelligence (ICRI-CI), and by the European Research Council (TheoryDL
project). 

\bibliography{mybib}
\bibliographystyle{plain}

\clearpage
\begin{center}
\Large
Supplementary material for \\
{\bf Weight Sharing is Crucial to Succesful Optimization}
\end{center}

\appendix
\section{Learning $g^*$ for Small $k$}\label{app:easy_g}

For simplicity of the argument, we consider the problem of learning a
function $g^* : \{\pm 1\}^k \to \reals$. The argument can be easily
extended to the case in which the domain of $g^*$ is $[-1,1]^k$ under
an additional Lipschitzness assumption, which is the case if the
activation of $h^*$ is, for example, the $\tanh$ function. 

Consider an arbitrary distribution, $D$, over $\{\pm 1\}^k$. Using
Lemma 19.2 in \cite{MLbook} we have that if $m > 2^k/\epsilon$ then,
in expectation over the choice of the sample, the probability mass of
vectors in $\{\pm 1\}^k$ that does not belong to my sample is at most
$\epsilon$. Therefore, finding a function $g$ that agrees with all the
points in the sample is sufficient to guarantee that $g$ agrees with
$g^*$ on all but an $\epsilon$-fraction of the vectors in
$\{\pm 1\}^k$.

Next, consider the problem of fitting a sample
$(x_1,g^*(x_1)),\ldots,(x_m,g^*(x_m))$ using a one-hidden-layer
network. Several papers have shown (e.g. \cite{livni2014computational,safran2016quality,choromanska2015loss,soudry2016no}) that if the number of
neurons in the hidden layer is at least $m$ than the optimization
surface is ``nice'' (in particular, no spurious local
minima). Furthermore, by a simple random embedding argument, it can be
shown that if we randomly pick the weight of the first layer and then
freeze them, then with high probability, there are weights for the
second layer for which the error is $0$ on all the training
examples. Learning only the second layer is a convex optimization
problem. 

Combining all the above we obtain that there is a procedure that runs
in time $\poly(2^k,1/\epsilon)$ that learns $g^*$ to accuracy
$\epsilon$. Note that since $k$ is small (in our experiment, we used
$k=5$), the term $2^k$ is very reasonable. This stands in contrast to
the term $d^k$, appearing in our lower bound for learning
$h^*$. Taking $d = 75, k=5$ (as in our experiment), the value of $d^k$
is huge.

\section{Proofs of \secref{sec:sum_cos}}\label{app:sec:sum_cos}
\subsection{Proofs of \secref{sec:hard_FC_cos}}\label{app:sec:hard_FC_cos}

\begin{proof}{\bf\label{proof:thm:cos_fc}of \thmref{thm:cos_fc}:}

Firstly, we note that by definition, $F(\bw)$ equals
\begin{align*}
&\frac{c_k^2}{2}\E\left[(\bw_1^\top\bx_1-\bu_0^{\top}\bx_1)\right]+\frac{1}{2}\cdot
\E\left[\left(\cos\left(\sum_{i=1}^{k}\bw_i^\top\bx_i\right)-\cos\left(\sum_{i=1}^{k}\bu_0^{\top}\bx_i\right)\right)^2\right]\\
&~~~~~~~~~~~+\E\left[(\bw_1^\top\bx_1-\bu_0^{\top}\bx_1)\left(\cos\left(\sum_{i=1}^{k}\bw_i^\top\bx_i\right)-\cos\left(\sum_{i=1}^{k}\bu_0^{\top}\bx_i\right)\right)\right]
\\
&=\frac{c_k^2}{2}(\bw_1-\bu_0)^\top\E[\bx_1\bx_1^\top](\bw_1-\bu_0)+\frac{1}{2}\cdot\E\left[
\left(\cos\left(\sum_{i=1}^{k}\bw_i^\top\bx_i\right)-\cos\left(\sum_{i=1}^{k}\bu_0^{\top}\bx_i\right)\right)^2\right]+0,
\end{align*}
where we used the facts that $(\bx_1,\ldots,\bx_k)$ are symmetrically 
distributed (hence take any value as well as its negative with equal 
probability) and that cosine is an even function. We get that
\begin{equation}\label{eq:Fdef2}
F(\bw) = \frac{c_k^2}{2}\E\left[\left(\bw_1^\top\bx_1-\bu_0\bx_1\right)^2\right]
+\frac{1}{2}\E_{\bz}\left[\left(\cos\left(\bw^\top \bz\right)
-\cos\left(\bar{\bu}_0^\top\bz\right)\right)^2\right],
\end{equation}
where $\bz$ is a standard Gaussian vector in $\reals^{kd}$. We prove
the following useful lemma:
\begin{lemma}\label{lem:gradient}
	Given some $b>0$, and assuming $\bz$ has a standard Gaussian distribution 
	in $\reals^d$, the gradient of $\frac{1}{2}\cdot\E_{\bz}\left[
	\left(\cos(b\cdot\bw^\top\bz)-\cos(b\cdot 
	\bar{\bu}_0^{\top}\bz)\right)^2\right]$ w.r.t. $\bw$ equals
	\[	
	-\frac{b}{2}\cdot\left(\phi(2b\bw)+\phi(b(\bw-\bu_0))-\phi(b(\bw+\bu_0))\right),
	\]
	where
	\[
	\phi(\ba) = \exp\left(-\frac{\norm{\ba}^2}{2}\right)\cdot\ba.
	\]
\end{lemma}
\begin{proof}
	A straightforward calculation reveals that the gradient equals
	\begin{align*}
	&-b\cdot\E\left[\left(\cos(b\cdot\bw^\top\bz)-\cos(b\cdot 
	\bar{\bu}_0^{\top}\bz)\right)\cdot\sin(b\cdot\bw^\top\bz)\bz\right]\\
	&=-\frac{b}{2}\left(\E[\sin(2b\cdot\bw^\top\bz)\bz]+\E[\sin(b(w-\bar{\bu}_0)^\top\bz)\bz]
	-\E[\sin(b(\bw+\bu_0)^\top\bz)\bz]\right).
	\end{align*}
\[
-\frac{b^2}{2}\left(2\exp\left(-2b^2\norm{\bw}^2\right)\bw
+\exp\left(-b^2\norm{\bw-\bu_0}^2\right)(\bw-\bu_0)
-\exp\left(-b^2\norm{\bw+\bu_0}^2\right)(\bw+\bu_0)\right).
\]
	We now argue that for any vector $\ba$, 
	\begin{equation}\label{eq:prodexp}
	\E[\sin(\ba^\top\bz)\bz] ~=~ \phi(\ba),
	\end{equation}
	from which the lemma follows. To see this, let $\bz_{\ba} = 
	\frac{\ba^\top \bz}{\norm{\ba}^2}\cdot\ba$ be the component of $\bz$ in the 
	direction of $\ba$, and let $\bz_{\perp \ba}=\bz-\bz_{\ba}$ be the 
	orthogonal component. Then we have
	\[
	\E[\sin(\ba^\top\bz)\bz] = \E[\sin(\ba^\top \bz_{\ba})\bz_{\ba}]+
	\E[\sin(\ba^\top \bz_{\ba})\bz_{\perp\ba}].
	\]
	The first term equals 
	$\frac{\ba}{\norm{\ba}^2}\cdot\E[\sin(\ba^\top\bz)\ba^\top\bz]$, or 
	equivalently $\frac{\ba}{\norm{\ba}}\cdot \E_{y}[\sin(\norm{\ba}y)y]$, 
	where $y$ has a standard Gaussian distribution. As to the 
	second term, we have that 
	$\bz_{\ba}$ and $\bz_{\perp\ba}$ are statistically independent (since $\bz$ 
	has a standard Gaussian distribution), so the term equals $\E[\sin(\ba^\top 
	\bz_{\ba})]\E[\bz_{\perp\ba}]=0$. Overall, we get that the expression above 
	equals
	\[
	\frac{\ba}{\norm{\ba}}\cdot 
	\E_{y}[\sin(\norm{\ba}y)y]~=~\frac{\ba}{\norm{\ba}}\cdot 
	\int_{y=-\infty}^{\infty}\sin(\norm{\ba}y)y\cdot 
	\frac{1}{\sqrt{2\pi}}\exp\left(-\frac{y^2}{2}\right)dy.
	\]
	Using integration by parts and the fact that $\int_y 
	\cos(by)\exp(-ay^2)=\sqrt{\frac{\pi}{a}}\exp(-b^2/4a)$ 
	(see \cite[equation 15.73]{spiegel1968schaum}), this equals
	\begin{align*}
	&=\frac{\ba}{\norm{\ba}\sqrt{2\pi}}\cdot\left(-\int_{y=-\infty}^{\infty}\sin(\norm{\ba}y)\exp\left(-\frac{y^2}{2}\right)dy
	+\norm{\ba}\int_{y=-\infty}^{\infty}\cos(\norm{\ba}y)\exp\left(-\frac{y^2}{2}\right)dy\right)\\
	&=\frac{\ba}{\sqrt{2\pi}}\cdot\left(0+\sqrt{2\pi}\exp\left(-\frac{\norm{\ba}^2}{2}\right)\right),
	\end{align*}
	from which \eqref{eq:prodexp} follows.
\end{proof}


Now, let us consider the partial derivative of this function w.r.t. 
$\hat{\bw}:= 
(\bw_2,\ldots,\bw_k)$. Using \lemref{lem:gradient}, and letting 
$\hat{\bu}_0=(\bu_0,\ldots,\bu_0)$ to be concatenation of $(k-1)$ copies of 
$\bu_0$, we get that this partial derivative equals
\[	
-\frac{1}{2}\left(
\exp\left(-\frac{\norm{\bw}^2}{2}\right)\hat{\bw}
+\exp\left(-\frac{\norm{\bw-\bar{\bu}_0}^2}{2}\right)\cdot(\hat{\bw}-\hat{\bu}_0)
-\exp\left(-\frac{\norm{\bw+\bar{\bu}_0}^2}{2}\right)\cdot(\hat{\bw}-\hat{\bu}_0)
\right),
\]
with norm at most
\[
\frac{1}{2}\left(
\exp\left(-\frac{\norm{\hat{\bw}}^2}{2}\right)\norm{\hat{\bw}}
+\exp\left(-\frac{\norm{\hat{\bw}-\hat{\bu}_0}^2}{2}\right)\cdot
\norm{\hat{\bw}-\hat{\bu}_0}
+\exp\left(-\frac{\norm{\hat{\bw}+\hat{\bu}_0}^2}{2}\right)\cdot
\norm{\hat{\bw}-\hat{\bu}_0}\right).
\]
Now, if $\norm{\hat{\bw}}\in
\left[\frac{\sqrt{k-1}}{3}\norm{\bu_0},\frac{\sqrt{k-1}}{2}\norm{\bu_0}\right]= 
\left[\frac{\norm{\hat{\bu}_0}}{3},\frac{\norm{\hat{\bu}_0}}{2}\right]$, as 
implied by the assumption stated in the theorem, it is 
easily verified that $\norm{\hat{\bw}-\hat{\bu}_0}^2$ as well 
as $\norm{\hat{\bw}-\hat{\bu}_0}^2$ are at least $\norm{\hat{\bu}_0}^2/3$, 
whereas $\norm{\hat{\bw}-\hat{\bu}_0},\norm{\hat{\bw}+\hat{\bu}_0}$ are at most 
$2\norm{\hat{\bu}_0}$. Moreover, 
$\frac{\norm{\hat{\bu}_0}}{\sqrt{k}\norm{\bu_0}} = \sqrt{1-\frac{1}{k}}\in 
\left[\frac{1}{\sqrt{2}},1\right]$.Thus, 
the displayed equation above can be upper bounded by
$c_1\sqrt{k}\norm{\bu_0}\exp(-c_2 k \norm{\bu_0}^2)$ for some numerical 
constants $c_1,c_2$. 

To prove the second part of the theorem, we rely on the following lemma:

\begin{lemma}
	\[
	\E\left[\left(\cos(\bw^\top\bx)-\cos(\bv^\top\bx)\right)^2\right] ~\geq~ 
	1-\exp(-\norm{\bw-\bv}^2/2)-\exp(-\norm{\bw+\bv}^2/2).
	\]
\end{lemma}
\begin{proof}
	Expanding the square and using standard trigonometric identities, we have 
	that the left hand side equals
	\begin{align*}
	&\E\left[\cos^2(\bw^\top\bx)\right]+\E\left[\cos^2(\bv^\top\bx)\right]-2\E\left[\cos(\bw^\top\bx)\cos(\bv^\top\bx)\right]\\
	&=1+\frac{1}{2}\E[\cos(2\bw^\top\bx)]+\frac{1}{2}\E[\cos(2\bv^\top\bx)]
	-\E\left[\cos((\bw-\bv)^\top\bx)\right]-\E\left[\cos((\bw+\bv)^\top\bx)\right].
	\end{align*}
	Since for any vector $\bz$, $\E[\cos(\bz^\top\bx)]=\E_y[\cos(\norm{\bz}y)]$ 
	where $y$ has a 
	standard normal distribution on $\reals$, and this in turn equals 
	$\frac{1}{\sqrt{2\pi}}\int 
	\cos(\norm{\bz}y)\exp(-y^2/2)=\exp(-\norm{\bz}^2/2)$
		(see \cite[equation 15.73]{spiegel1968schaum}), the above equals
	\[
	1+\frac{1}{2}\exp\left(-2\norm{\bw}^2\right)+\frac{1}{2}\exp\left(-2\norm{\bv}^2\right)
	-\exp\left(-\frac{\norm{\bw-\bv}^2}{2}\right)-\exp\left(-\frac{\norm{\bw+\bv}^2}{2}\right),
	\]
	from which the result follows.
\end{proof}

Using this lemma, and definition of $F(\bw)$ in \eqref{eq:Fdef2}, we have
\begin{align*}
2(F(\bw)-F(\bar{\bu}_0)) &~\geq~ 
1-\exp(-\norm{\bw-\bu_0}^2/2)-\exp(-\norm{\bw+\bu_0}^2/2)\\
&~\geq~
1-\exp(-\norm{\hat{\bw}-\hat{\bu}_0}^2/2)-\exp(-\norm{\hat{\bw}+\hat{\bu}_0}^2/2).
\end{align*}
Moreover, assuming that $\norm{\hat{\bw}} \leq 
\frac{\sqrt{k-1}}{2}\norm{\bu_0}=\frac{\norm{\hat{\bu}_0}}{2}$ (as implied by 
the assumption stated in the theorem), we have that 
$\norm{\hat{\bw}-\hat{\bu}_0}^2$ 
as well 
as $\norm{\hat{\bw}-\hat{\bu}_0}^2$ are at least $\norm{\hat{\bu}_0}^2/4$, 
which in turn equals $(k-1)\norm{\bu_0}^2/4\geq k\norm{\bu_0}^2/8$. Plugging to 
the above, the result follows. 
\end{proof}

\subsection{Proofs of
  \secref{sec:easy_WS_cos}}\label{app:sec:easy_WS_cos}

\begin{proof}{\bf\label{proof:thm:cos_conv} of \thmref{thm:cos_conv}:}
The fact that $\bu_0$ minimizes $F(\cdot)$ is immediate from the definition. 
Also, given that $F(\cdot)$ is strongly convex and satisfies the eigenvalue 
condition stated in the theorem, the convergence bound for gradient descent 
follows from standard results (see \cite{nesterov2004introductory}). Thus, it 
remains to prove the strong convexity and eigenvalue bounds.

To get these bounds, we use the same calculations as in the beginning
of the proof of \thmref{thm:cos_fc}, to rewrite $F(\bw)$ as
\begin{equation}\label{eq:acobj}
F(\bw) = \frac{c_k^2}{2}\norm{\bw-\bu_0}^2+\frac{1}{2}\cdot\E_{\bz}\left[
\left(\cos(\sqrt{k}\cdot\bw^\top\bz)-\cos(\sqrt{k}\cdot 
\bu_0^{\top}\bz)\right)^2\right],
\end{equation}
where $\bz$ has a standard Gaussian distribution in $\reals^d$, and
using the fact $\E[\bx_i\bx_i^\top]=I$ is the identity matrix, and
$\sum_{i=1}^{k}\bx_i$ is distributed as a Gaussian with mean $0$ and
variance $kI$.



The following lemma will be useful in computing the Hessian of $F$:

\begin{lemma}\label{lem:hessian}
	Given some $b>0$, and assuming $\bz$ has a standard Gaussian distribution 
	in $\reals^d$, the Hessian of $\frac{1}{2}\cdot\E_{\bz}\left[
	\left(\cos(b\cdot\bw^\top\bz)-\cos(b\cdot 
	\bw^{*\top}\bz)\right)^2\right]$ w.r.t. $\bw$ has a spectral norm upper 
	bounded by $2b^2+b$.
\end{lemma}
\begin{proof}
	Differentiating the gradient as defined in \lemref{lem:gradient}, and 
	noting that $\frac{\partial}{\partial \bw}\phi(c_k\bw) = 
		c_k\exp(-c_k^2\norm{\bw}^2/2)(I+c_k^2\bw\bw^\top)$ for any $c_k$, we 
		get that 
		the Hessian equals
	\begin{align*}
	\frac{b^2}{2}&\Big(-
	2\exp\left(-2b^2\norm{\bw}^2\right)(I-2b\bw\bw^\top)-
	\exp\left(-\frac{b^2\norm{\bw-\bu_0}^2}{2}\right)
	(I-b(\bw-\bu_0)(\bw-\bu_0)^\top)\\
	&+
	\exp\left(-\frac{b^2\norm{\bw+\bu_0}^2}{2}\right)
	(I-b(\bw+\bu_0)(\bw+\bu_0)^\top)\Big)\\
	&= 
	\frac{b^2}{2}\left(-2\exp(-2b^2\norm{\bw}^2)-\exp\left(-\frac{b^2\norm{\bw-\bu_0}^2}{2}\right)
	+\exp\left(-\frac{b^2\norm{\bw+\bu_0}^2}{2}\right)\right)I\\
	&~~~+ \frac{b}{2}\Big(4b^2\exp(-2b^2\norm{\bw}^2)\bw\bw^\top
	+b^2 
	\exp\left(-\frac{b^2\norm{\bw-\bu_0}^2}{2}\right)(\bw-\bu_0)(\bw-\bu_0)^\top\\
	&~~~~~~~~~~~~-b^2
	\exp\left(-\frac{b^2\norm{\bw+\bu_0}^2}{2}\right)(\bw+\bu_0)(\bw+\bu_0)^\top
	\Big).
	\end{align*}
	Therefore, its spectral norm is at most
	\begin{align*}
	\frac{b^2}{2}&\cdot 3+\frac{b}{2}\Big(2\cdot 
	\exp(-2b^2\norm{\bw}^2)\left(2b^2\norm{\bw}^2\right)+\exp\left(-\frac{b^2
	\norm{\bw-\bu_0}^2}{2}\right)\left(b^2\norm{\bw-\bu_0}^2\right)\\
	&+\exp\left(-\frac{b^2
		\norm{\bw+\bu_0}^2}{2}\right)\left(b^2\norm{\bw+\bu_0}^2\right)\Big).
	\end{align*}
	Using the easily-verified fact that $\max_{z\geq 0}\exp(-z)z = \exp(-1)$, 
	we get that the above is at most
	\[
	\frac{3b^2}{2}+\frac{b}{2}\exp(-1)(2+1+1) = \frac{3b^2}{2}+2\exp(-1)b < 
	2b^2+b
	\]
	as required.
\end{proof}

We are now in place to prove our theorem. Applying \lemref{lem:hessian} 
and using the definition of the objective function $F(\bw)$ at 
\eqref{eq:acobj}, we get that $\nabla^2 F(\bw)$ has eigenvalues in 
the range $[c_k^2/2-(2k+\sqrt{k}),c_k^2/2+(2k+\sqrt{k})]$. Since 
$c_k=3\sqrt{k}$, we get that every eigenvalue of the Hessian is 
lower bounded by $\frac{9k}{2}-2k-\sqrt{k} \geq \frac{9k}{2}-3k = 
\frac{3}{2}k$, and upper bounded by $\frac{9k}{2}+2k+\sqrt{k} \leq 
\frac{9k}{2}+3k = \frac{15}{2}k$. Since $k\geq 1$, this implies that the 
Hessian is positive definite everywhere (with minimal eigenvalue at least 
$3/2$), hence $F$ is strongly convex. 
Moreover, 
\[
\frac{\max_{\bw}\lambda_{\max}(\nabla^2 
F(\bw))}{\min_{\bw}\lambda_{\min}(\nabla^2 F(\bw))}\leq \frac{15k/2}{3k/2} = 5
\]
as required.
\end{proof}

\section{Proofs of \secref{sec:sum_parities}}
\subsection{Proofs of \secref{sec:non_degenerate}}\label{app:sec:non_degenerate}
\begin{proof}{\bf\label{proof:lem:non_degenerate} of \lemref{lem:non_degenerate}:}
  It suffices to show that
  $V_{\sigma}(\bu_0,\bu_0)=\E_{\bx_1} \sigma(\bu_0^\top\bx_1)^2 >
  \left(1-\frac{1}{k}\right)$, for in that case, by the independence
  of $\bx_1,\ldots,\bx_k$ we have
\[
\E_\bx\left[(p_{\bu_0}^{(k)}(\bx))^2\right] = (\E_{\bx_1}[\sigma(\bu_0^\top\bx_1)^2])^k > \left(
  1-\frac{1}{k} \right)^{k} > \frac{1}{4}.
\]
To show that $V_{\sigma}(\bu_0,\bu_0) > 1-\frac{1}{k}$, note that from
\lemref{lem:williams} we have that
$V_{\sigma}(\bu_0,\bu_0) =
\frac{2}{\pi}\sin^{-1}\left(\frac{2\|\bu_0\|^2}{1+2\|\bu_0\|^2}
\right)$. Denote by $f(a)$ the value for which
$\frac{2}{\pi}\sin^{-1}\left(\frac{f(a)}{1+f(a)} \right) = 1 - a$. By
standard algebraic manipulations we have that
\[
f(a) = \frac{ \sin(\tfrac{\pi}{2}(1-a)) } { 1 -
  \sin(\tfrac{\pi}{2}(1-a)) } ~.
\]
Using Taylor's theorem we have that there exists $\xi \in
[\tfrac{\pi}{2}(1-a),\tfrac{\pi}{2}]$ such that
\[
\sin(\tfrac{\pi}{2}(1-a)) = 1 - \frac{1}{2} \left(\frac{\pi}{2}
  a\right)^2 + \frac{1}{6} \cos(\xi)  \left(\frac{\pi}{2}
  a\right)^3
=
1 - \left(\frac{\pi}{2} a\right)^2 \left[
  \frac{1}{2} - \frac{1}{6} \cos(\xi) \frac{\pi}{2} a \right] ~.
\]
It follows that
\[
f(a) = \frac{1}{\left(\frac{\pi}{2} a\right)^2 \left[
  \frac{1}{2} - \frac{1}{6} \cos(\xi) \frac{\pi}{2} a \right] } - 1
\le \frac{1}{\left(\frac{\pi}{2} a\right)^2 \left[
  \frac{1}{2} - \frac{1}{6} \cos(\xi) \frac{\pi}{2} a \right] } \le
\frac{3}{\left(\frac{\pi}{2} a\right)^2 }  ~,
\] 
where in the last inequality we assume that 
$a \in [0,1/2]$. 
Taking $a = 1/k$ and noting that $V_\sigma(\bu_0,\bu_0)$ monotonically
increases with $\|\bu_0\|$ we conclude our proof.
\end{proof}

\subsection{Proofs of \secref{sec:exact_expressions}}\label{app:sec:exact_expressions}

\begin{proof}{\bf \label{proof:lem:Fsum_sumF}
of \lemref{lem:Fsum_sumF}:} We start by expanding $F^{(1,k)}(\bw)$:
\begin{align*}
  &F^{(1,k)}(\bw)\\ 
  =&\E_\bx \left[ (p_\bw(\bx) - p_{\bu_0}(\bx))^2\right]\\
  =&\E_\bx \left[ (p^{(1)}_\bw(\bx) - p^{(1)}_{\bu_0}(\bx)+
     p^{(k)}_\bw(\bx) - p^{(k)}_{\bu_0}(\bx))^2\right]\\
  =&\E_\bx \left[ (p^{(1)}_\bw(\bx) - p^{(1)}_{\bu_0}(\bx))^2\right]+ \E_\bx
     \left[ (p^{(k)}_\bw(\bx) - p^{(k)}_{\bu_0} (\bx))^2\right] + 2\E_\bx
     \left[(p^{(1)}_\bw(\bx) - p^{(1)}_{\bu_0} (\bx)) (p^{(k)}_\bw(\bx) -
     p^{(k)}_{\bu_0} (\bx))\right]\\
  =&F^{(1)}(\bw)+F^{(k)}(\bw) + 2\E_\bx \left[(p^{(1)}_\bw(\bx) - p^{(1)}_{\bu_0} (\bx)) (p^{(k)}_\bw(\bx) - p^{(k)}_{\bu_0} (\bx))\right]\\
\end{align*}
We next show that $\E_\bx \left[(p^{(1)}_\bw(\bx) -
  p^{(1)}_{\bu_0} (\bx)) (p^{(k)}_\bw(\bx) -
  p^{(k)}_{\bu_0} (\bx))\right]=0$. Since $\sigma$ is anti-symmetric
and $\bx_k$ is normal, we have that for every vector $\bw'$, $\E_{\bx_k} \sigma(\bw'^\top \bx_k) =
  0$. By the
  independence of the $\bx_i$'s, 
\begin{align*}
  \E_\bx \left[p^{(1)}_\bw(\bx)\cdot p^{(k)}_\bw(\bx)\right]
  =&\E_{\bx_1} p^{(1)}_\bw(\bx)\E_{\bx_{2...k-1}} \left[
     \prod_{i=1}^{k-1} \sigma(\bw_i^\top \bx_i) \right] \E_{\bx_k}\left[
     \sigma(\bw_k^\top \bx_k)  \right] = 0
\end{align*}
The same argument holds for the other terms, and the result follows.
\end{proof}

\begin{proof}{\bf\label{proof:lem:pg_pmgm} of \lemref{lem:pg_pmgm}}
We start with the FC setting. Since
  $\sigma$ is antisymmetric we clearly have that
\[
p^{(k)}_\bw(\bx) = \prod_{l=1}^k \sigma(\bw_l^\top \bx_l) = 
(-1)^k \prod_{l=1}^k \sigma(-\bw_l^\top \bx_l) = (-1)^k
p^{(k)}_\bw(-\bx) ~.
\]
Next, for every $l$, let $\bg^{(k)}_{\bw'_l}$ be the derivative w.r.t. the weights
  corresponding to the $l$'th input instance, then
\[
\bg^{(k)}_{\bw'_l}(\bx) = \left(\prod_{j \neq l}\sigma(\bw'_j\bx_j)
\right)\sigma'(\bw'_l\bx_l)\bx_l
= (-1)^{k}  \left(\prod_{j \neq l}\sigma(-\bw'_j\bx_j)
\right)\sigma'(-\bw'_l\bx_l)(-\bx_l)
= (-1)^{k} \bg^{(k)}_{\bw'_l}(-\bx) ~,
\]
where we used the fact that $\sigma'$ is symmetric. The claim follows
because $(-1)^{2k} = 1$. Finally, for the WS setting, we can think of
$\bw'$ as being $k$ copies of $\bw'_0$ and then, by standard derivative rules, $\bg^{(k)}_{\bw'_0}=\sum_l
\bg^{(k)}_{\bw'_l}$, from which the claim
follows.
\end{proof}

\begin{proof}{\bf\label{proof:lem:exact_gradient}
 of \lemref{lem:exact_gradient}:}. For the first term:
\begin{align*}
  \ba_l := &\E_{\bx}\left[p^{(k)}_{\bw}(\bx)\bg^{(k)}_l(\bx)\right]\\
=&\frac{1}{2}\E_{\bx:~\bw_l^\top\bx_l>0}\left[p^{(k)}_{\bw}(\bx)\bg^{(k)}_{l}(\bx)\right]
  + \frac{1}{2}\E_{\bx:~\bw_l^\top\bx_l<0}\left[p^{(k)}_{\bw}(\bx)\bg^{(k)}_{l}(\bx)\right]\\
\overset{(0)}{=}&\frac{1}{2}\E_{\bx:~\bw_l^\top\bx_l>0}\left[p^{(k)}_{\bw}(\bx)\bg^{(k)}_{l}(\bx)\right]
  +
   \frac{1}{2}\E_{\bx:~\bw_l^\top\bx_l<0}\left[p^{(k)}_{\bw}(-\bx)\bg^{(k)}_{l}(-\bx)\right]\\
\overset{(1)}{=}&\frac{1}{2}\E_{\bx:~\bw_l^\top\bx_l>0}\left[p^{(k)}_{\bw}(\bx)\bg^{(k)}_{l}(\bx)\right]
  + \frac{1}{2}\E_{\bx:~\bw_l^\top\bx_l>0}\left[p^{(k)}_{\bw}(\bx)\bg^{(k)}_{l}(\bx)\right]\\
=&\E_{\bx:~\bw_l^\top\bx_l>0}\left[p^{(k)}_{\bw}(\bx)\bg^{(k)}_{l}(\bx)\right]\\
  =&\E_{\bx:~\bw_l^\top\bx_l>0}\left[(\prod_{j}\sigma(\bw_j^\top\bx_j))(\prod_{j\neq
     l}\sigma(\bw_j^\top\bx_j))\sigma'(\bw_l^\top\bx_l)\bx_l\right]\\  
\overset{(2)}{=}&\E_{\bx_{[k]\setminus \{l\}}}\left[\prod_{j\neq l}\sigma^2(\bw_j^\top\bx_j)\right]\cdot\E_{\bx_l:~\bw_l^\top\bx_l>0}\left[\sigma(\bw_l^\top\bx_l)\sigma'(\bw_l^\top\bx_l)\bx_l\right]\\
\overset{(3)}{=}&\left(\prod_{j\neq l}V_{\sigma}(\bw_j, \bw_j)\right)\cdot\E_{\bx_l:~\bw_l^\top\bx_l>0}\left[\sigma(\bw_l^\top\bx_l)\sigma'(\bw_l^\top\bx_l)\bx_l\right]\\
\overset{(4)}{=}&\left(\prod_{j\neq l}V_{\sigma}(\bw_j,
                  \bw_j)\right)\cdot  c_1(\bw_l)\bw_l ~.
\end{align*}
In the above, $(0)$ is from \lemref{lem:pg_pmgm}, $(1)$ uses symmetry of the
probability of $\bx$, and $(2)$, $(3)$ follow from the fact all
$\bx_l$ are i.i.d.. To see why $(4)$ is true, and why $c_1(\bw_l) \ge
0$, note that we can assume w.l.o.g. that $\bw_l = (\|\bw_l\|,0,\ldots,0)$
(because $\bx_l$ is Gaussian), and in this case it is clear that the
first coordinate of
$\E_{\bx_l:~\bw_l^\top\bx_l>0}\left[\sigma(\bw_l^\top\bx_l)\sigma'(\bw_l^\top\bx_l)\bx_l\right]$
is positive while the rest of the coordinates are zero. 

For the second part of the lemma, similar arguments give:
\begin{align*}
  \bb_l:=&\E_{\bx}\left[p^{(k)}_{\bu_0}(\bx)\bg^{(k)}_l(\bx)\right]
  \overset{(0)}{=}\E_{\bx:~\bu_0^\top\bx_l>0}\left[p^{(k)}_{\bu_0}(\bx)\bg^{(k)}_l(\bx)\right]\\
  =&\E_{\bx:~\bu_0^\top\bx_l>0}\left[\left(\prod_{j\neq
     l} \sigma(\bu_0^\top\bx_j) \sigma(\bw_j^\top\bx_j) \right) \, \sigma(\bu_0^\top\bx_l)\,\sigma'(\bw_l^\top\bx_l)\bx_l\right] \\
  =&\left(\prod_{j\neq
     l} V_\sigma(\bu_0,\bw_j)\right)\cdot
  \E_{\bx:~\bu_0^\top\bx_l>0}\left[ \sigma(\bu_0^\top\bx_l)\,\sigma'(\bw_l^\top\bx_l)\bx_l\right]
\end{align*}
where $(0)$ is a similar transition to that done for $\ba_l$, using
\lemref{lem:pg_pmgm}.  Using again the normality of $\bx$, it is easy
to see that 
$\tilde{\bb}_l := \E_{\bx:~\bu_0^\top\bx_l>0}\left[
  \sigma(\bu_0^\top\bx_l)\,\sigma'(\bw_l^\top\bx_l)\bx_l\right]$ is in
the span of $\{\bu,\bw_l\}$.
\end{proof}

\subsection{Proofs of \secref{sec:hard_FC}}\label{app:sec:hard_FC}

\begin{proof}{\bf \label{proof:lem:chernoff_FC} of
  \lemref{lem:chernoff_FC}:} 
The random variable, $\frac{1}{d}\,\bw_l^\top \bu_0$ is an average of $d$ random
variables, each of which distributed uniformly over $\{\pm c
\bu_{0,i}\}$, and its expected value is zero. Hence, by Hoeffding's
inequality,
\[
\Pr\left[ \frac{|\bw_l^\top \bu_0|}{\|\bw_l\|\,\|\bu_0\|} > d^{-1/3} \right]
= 
\Pr\left[ \frac{1}{d} \, |\bw_l^\top \bu_0| >
  \|\bw_l\|\,\|\bu_0\|\,d^{-4/3} \right] ~\le~ 2
\exp\left(- 0.5\, d^{1/3}  \right) ~.
\]
Applying a union bound over the $k$ weight vectors, we conclude our
proof. 
\end{proof}

\begin{proof}{\bf\label{proof:lem:small_Verf} of \lemref{lem:small_Verf}:}
By the symmetry of $V_\erf$ we can assume w.l.o.g. that
  $\bu_0^\top\bw_j>0$, and then $V_\erf(\bw_j,\bu_0)  > 0$. 
We can rewrite
\begin{align*}
V_\erf(\bw_j,\bu_0)~=&~\frac{2}{\pi}\sin^{-1}\left(\frac{2\bu_0^\top\bw_j}{\sqrt{1+2\|\bu_0\|^2}\sqrt{1+2\|\bw_j\|^2}}\right)\\
=&~\frac{2}{\pi}\sin^{-1}\left( \frac{\bu_0^\top \bw_j}{\|\bu_0\|\,\|\bw_j\|} \cdot \frac{\sqrt{2}
        \|\bu_0\|}{\sqrt{1+(\sqrt{2}\|\bu_0\|)^2}} \cdot
     \frac{\sqrt{2}\,\|\bw_j\|}{\sqrt{1+(\sqrt{2}\|\bw_j\|)^2}}\right)
\end{align*}
The function $f(a) = \frac{a}{\sqrt{1+a^2}}$ is monotonically
increasing over $a \in [0,\infty)$, where $f(a)=0$ and $f(a) \to 1$ as
$a \to \infty$. Therefore, all the three terms in the argument of the inverse sign are
in $[0,1]$. If the first condition in the theorem holds then the
first term makes the argument of the inverse sign at most
$c$. If the second condition in the theorem holds then, since
$f(a) \le a$, we have that the third term is at most $c$. 
The claim now follows immediately because for every $c \in [0,1]$ we
have $c \le \sin(c)$.
\end{proof}

\begin{proof}{\bf\label{proof:lem:non_degenerate_large_loss} of
  \lemref{lem:non_degenerate_large_loss}:} By \lemref{lem:small_Verf},
  $V_\erf(\bw_j,\bu_0)<\frac{1}{16}$. From
  \lemref{lem:Fsum_sumF}, we have that $F^{(1,k)}\geq F^{(k)}$. Also,
  by \lemref{lem:non_degenerate},
  $\E_\bx\left[(p^{(k)}_{\bu_0}(\bx))^2\right] > \frac{1}{4}$. Thus,
\begin{align*}
  F^{(1,k)}~\geq~ F^{(k)}=&~\E_\bx\left[(p^{(k)}_{\bw}(\bx)-p^{(k)}_{\bu_0}(\bx))^2\right]\\
  \geq&~
        \E_\bx\left[(p^{(k)}_{\bw}(\bx))^2\right]-2\E_\bx\left[p^{(k)}_{\bw}(\bx)
        p^{(k)}_{\bu_0}(\bx)\right] + \frac{1}{4}\\
\geq&~
        -2\E_\bx\left[p^{(k)}_{\bw}(\bx)
        p^{(k)}_{\bu_0}(\bx)\right] + \frac{1}{4}\\
  =&~-2\prod_lV_\erf(\bw_l,\bu_0) + \frac{1}{4}\\
\geq&~-2\frac{1}{16}\prod_{l\neq j}|V_\erf(\bw_l,\bu_0)| + \frac{1}{4}
  ~\geq~-\frac{1}{8} + \frac{1}{4} ~=~ \frac{1}{8}
\end{align*}
\end{proof}

\begin{proof}{\bf \label{proof:thm:main_FC_fail} of \thmref{thm:main_FC_fail}}
To simplify the notation throughout this proof, whenever we write $\bw_l$ we mean for $l \ge
2$. Recall that the gradient w.r.t. $\bw_l$ is equal to:
\[
\nabla_{\bw_l} F^{(1,k)}(\bw^{(t)}) = \left(\prod_{j\neq l}V_{\sigma}(\bw_j, \bw_j)\right)\cdot
c_1(\bw_l)\bw_l - \left(\prod_{j\neq
     l} V_\sigma(\bu_0,\bw_j)\right)\cdot \tilde{\bb}_l.
\]
Let $T = \lfloor \frac{1}{\eta k} d^{\frac{k}{3} - 3} \rfloor$. 
We firstly prove, by induction, that for all $t \in \{0,1,\ldots,T\}$, at least
one of the following holds for every $l$:
\begin{enumerate}
\item $\frac{|(\bw^{(t)}_l)^\top \bu_0|}{\|\bw^{(t)}_l\|\,\|\bu_0\|} \le
  d^{-1/3}+\eta t\cdot d^{-\frac{k-2}{3} + 2}$, 
\item at some $t'\leq t$, it held that $\|\bw^{(t')}_l\|<1/d$, and
  $\|\bw^{(t)}_l\|<1/d+\eta (t-t')\cdot d^{-(k-2)/3+1}$.
\end{enumerate}
This holds w.h.p. for $t=0$, by \lemref{lem:chernoff_FC}. 
For the
inductive step, note that if the claim holds for some $t \le T$, 
then by the definition of $T$ we have that $\eta t d^{-\frac{k-2}{3} +
  2} \le d^{-1/3}/k$. Hence, by \lemref{lem:small_Verf},  the
coefficient of $\tilde{\bb}_l$ is at most
\[
\left(1 + \frac{1}{k} \right)^{k-2} \cdot d^{-(k-2)/3} ~\le~ e\,
d^{-(k-2)/3} ~\le~ d^{-(k-2)/3 + 1} ~,
\]
where we assume that $d \ge e$. 
Moreover, it is easy to see that
$\|\tilde{\bb}_l\| < 1$.  In addition, it is easy to see that the coefficient of
$\bw^{(t)}_l$ in the first term of the gradient is
positive. Therefore, if the second assumption held for $l$ at time
$t$, then by observing the explicit expression of the gradient, we
obtain that the only term which can increase $\|\bw^{(t)}_l\|$ is
smaller in magnitude than $\eta \,d^{-(k-2)/3+1}$, and hence, the second
assumption holds for $t+1$. If on the other hand, the first assumption
held for $l$ at time $t$. If at $t+1$, its norm decreases below $1/d$,
we are done. Otherwise, note that the only term of the gradient which
can change the direction of $\bw^{(t)}_l$ is the second one, which is
again, smaller in magnitude than $\eta \,d^{-(k-2)/3+1}$. Now, since
$\|\bw^{(t+1)}_l\|>1/d$, we obtain that the change in angle,
$d_\alpha$, between $\bw_l,\bu_0$, in times $t,t+1$, satisfies
$d_\alpha<\tan^{-1}\left(\frac{\eta \,d^{-(k-2)/3+1}}{1/d}\right) = \tan^{-1}(\eta d^{-(k-2)/3+2})$. Therefore, denoting by $\alpha^{(t)}$ the angle at
time $t$, we have,
\begin{align*}
  |\frac{|(\bw^{(t+1)}_l)^\top \bu_0|}{\|\bw^{(t+1)}_l\|\,\|\bu_0\|} -
  \frac{|(\bw^{(t)}_l)^\top \bu_0|}{\|\bw^{(t)}_l\|\,\|\bu_0\|}
  |~=&~|\cos\alpha^{(t+1)}-\cos\alpha^{(t)}|\\
  \overset{(0)}{\leq}&~|\alpha^{(t+1)}-\alpha^{(t)}|
  \leq\tan^{-1}(\eta d^{-(k-2)/3+2}) \overset{(1)}{\leq} \eta d^{-(k-2)/3+2}
\end{align*}
where $(0)$ and $(1)$ are by the 1-Lipschitzness of $\cos$ and
$\tan^{-1}$. The inductive step follows immediately. From this proof,
combined with \lemref{lem:non_degenerate_large_loss}, we obtain that
for all $T=O(d^{k/4})$, $F^{(1,k)}(\bw^{(T)})>1/8$, as required.
\end{proof}

\subsection{Proofs of \secref{sec:easy_WS}}\label{app:sec:easy_WS}

\begin{proof}{\bf\label{proof:thm:phase1} of \thmref{thm:phase1}:}
  Firstly, we use the lemmas from \secref{sec:exact_expressions} to
  write the gradient as:
\begin{align*}
\nabla F^{(1,k)}(\bw_0) = &\nabla F^{(1)}(\bw_0)+\nabla F^{(k)}(\bw_0)\\
=&c_1(\bw_0)\bw_0-\tilde{\bb}_1+kV_{\sigma}(\bw_0,\bw_0)^{k-1}\cdot
   c_1(\bw_0)\bw_0-kV_\sigma(\bu_0,\bw_0)^{k-1}\cdot \tilde{\bb}_1\\
=&(1+kV_{\sigma}(\bw_0,\bw_0)^{k-1})c_1(\bw_0)\bw_0 -
   (1+kV_\sigma(\bu_0,\bw_0)^{k-1}) \tilde{\bb}_1\\
\end{align*}
Observe that a gradient update is the subtraction of the two terms
(scaled by $\eta$) from
$\bw_0$. Hence, the first term does not change the
angle $\alpha$. We further note that this update adds $\tilde{\bb}_1$
to $\bw_0$ with a positive coefficient. It is therefore sufficient to show that the second term has
positive inner product with $\bw_\perp := (-\cos(\alpha),\sin(\alpha))$,
see \figref{fig:phase1}.


Recall that $ \tilde{\bb}_1=
\E_{\bx_1:~\bu_0^\top\bx_1>0}\left[\sigma(\bu_0\bx_1)\sigma'(\bw_0\bx_1)\bx_1\right]$.
We note that for
every $\bx=(\theta,r)$ (in polar coordinates)
s.t. $\theta\in[0,\pi/2-\alpha]$, we can look at its reflection over
the $\bw_0$ axis, namely, $\tilde{\bx}=(\pi-2\alpha-\theta,r)$, and note
that:
\begin{itemize}
\item $\sigma'(\bw_0\bx)=\sigma'(\bw_0\tilde{\bx})$,
\item $\inner{\bx, \bw_\perp}=-\inner{\tilde{\bx}, \bw_\perp}$,
\item $\sigma(\bu_0\bx) <\sigma(\bu_0\tilde{\bx}) $.
\end{itemize}
Let $A_1$ be the event that $\theta\in[0,\pi-2\alpha]$. By the
above properties,
\begin{align*}
\E_{\bx\in A_1}\left[\inner{\sigma(\bu_0\bx)\sigma'(\bw_0\bx)\bx,
  \bw_\perp}\right] =& \E_{\bx\in
                       A_1~:~\bw_{\perp}^\top\bx<0}\left[\left(\sigma(\bu_0\bx)-\sigma(\bu_0\tilde{\bx})\right)\sigma'(\bw_0\bx)\inner{\bx,
                       \bw_\perp}\right] >0,
\end{align*}
since $\inner{\bx, \bw_\perp}<0$,
$\left(\sigma(\bu_0\bx)-\sigma(\bu_0\tilde{\bx})\right)<0$, and
$\sigma'(\bw_0\bx)\geq0$. Thus,
\[
\inner{\tilde{\bb}_1, \bw_\perp}\geq \Pr(A_1^c)\E_{\bx\in A_1^c}\left[\inner{\sigma(\bu_0\bx_1)\sigma'(\bw_0\bx_1)\bx_1, \bw_\perp}\right]
\]
Moreover, note that for all $\bx\in A_1^c$, the expression in the
expectation is non negative. Hence we obtain that
$\inner{\tilde{\bb}_1, \bw_\perp}\geq 0$ for all $\tilde{\bb}_1$, and
indeed, $\alpha^{(t+1)}\leq\alpha^{(t)}$. We shall now show a positive
lower bound.  Let $A_2$ be the event that
$\theta\in[\pi-\frac{3}{2}\alpha, \pi-\frac{1}{2}\alpha]$, and
$\frac{1}{\|\bu_0\|}\leq r/\sqrt{2} < \frac{2}{\|\bu_0\|}$. In the
angular aspect, these are, in words, elements with an angle of
$\alpha/2$ around $\bw_\perp$, see \figref{fig:phase1} for an
illustration. Firstly note that $A_2\subset A_1^c$, so from positivity
of the inner expression in the expectation:
\[
\inner{\tilde{\bb}_1, \bw_\perp}\geq \Pr(A_2)\E_{\bx\in A_2}\left[\inner{\sigma(\bu_0\bx_1)\sigma'(\bw_0\bx_1)\bx_1, \bw_\perp}\right]
\]
\begin{figure}
\begin{center}
\begin{tikzpicture}[scale=2]
\def\sin30{0.866};
\def\cos30{0.5};
\def\cosalphadiv2{0.965};
\def\sinalphadiv2{0.258};
\def\cosalpha3div2{0.707};
\def\sinalpha3div2{0.707};
\draw[help lines] (-3,0) grid (3,3);
\draw [blue,fill=cyan] (3,0) --(3,3)--
(-3*\cos30/\sin30,3)--(0,0)--(3,0);
\draw [very thick] (3,0)--(0,0)--(-3*\cos30/\sin30,3);
\draw [very thick, fill=yellow] (-0.33*\cosalphadiv2, 0.33*\sinalphadiv2) --
(-0.66*\cosalphadiv2,0.66*\sinalphadiv2) to [out=90-15, in=15*3+180]
(-0.66*\cosalpha3div2,0.66*\sinalpha3div2)--(-0.33*\cosalpha3div2, 0.33*\sinalpha3div2)
to [out=15*3+180, in=90-15](-0.33*\cosalphadiv2, 0.33*\sinalphadiv2);
\draw [->](0,0) --(0,3); \node [above] at (0,3) {$\bu_0$};
\draw [->](0,0) --(1,2); \node [right] at (1,2) {$\bw_0$};
\draw [->](0,0) --(-\sin30, \cos30); \node [above] at (-\sin30, \cos30) {$\bw_\perp$};
\node at (2.5,2.5) {$A_1$};
\node at (-0.5*\sin30, 0.5*\cos30) {$A_2$};
\end{tikzpicture}
\caption{An illustration of the defined vectors and events used in
  \thmref{thm:phase1}.}\label{fig:phase1}
\end{center}
\end{figure}
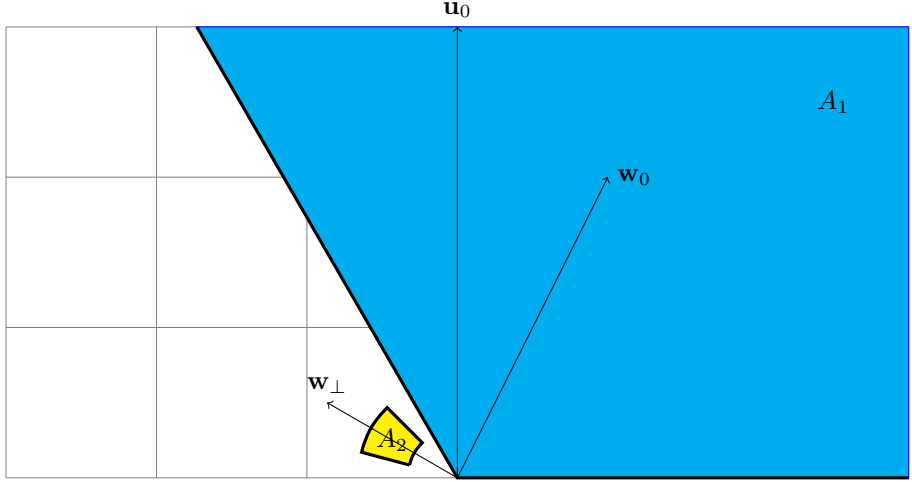
Firstly, let us lower bound $\Pr(A_2)$. In the angular aspect, it is
clear that
$\Pr\left(\theta\in[\pi-\frac{3}{2}\alpha,
  \pi-\frac{1}{2}\alpha]\right)=\frac{\alpha}{\pi}$. For each such
$\theta$, the $r/\sqrt{2}$ value of $\bx$ is distributed
$|\mathcal{N}(0,1)|$. It is clear from monotonicity of the Gaussian
$\mathrm{pdf}$ for positive $r$, that:
$ \Pr\left(\frac{1}{\|\bu_0\|}<r/\sqrt{2}<\frac{2}{\|\bu_0\|}\right)\geq
\frac{1}{\|\bu_0\|}\cdot
\frac{2}{\sqrt{2\pi}}e^{-\frac{2^2}{2 \|\bu_0\|^2}}>\frac{1}{5 \|\bu_0\|} $.  Therefore,
we continue:
\[
  \inner{\tilde{\bb}_1, \bw_\perp}\geq \frac{\alpha}{5\pi \|\bu_0\|}\min_{\bx\in
    A_2}\sigma(\bu\bx)\min_{\bx\in A_2}\sigma'(\bw\bx)\min_{\bx\in
    A_2}\inner{\bx,\bw_\perp}
\]
It is easy to see where each of the minimas is obtained, and hence we have:
\begin{align*}
  \inner{\tilde{\bb}_1, \bw_\perp}
\geq&\frac{\alpha}{5 \|\bu_0\|\pi}
\sigma\left(\|\bu_0\|\frac{\sin(\pi-\frac{\alpha}{2})}{\|\bu_0\|}\right)
\sigma'\left(\|\bw_0\|\frac{2 (\cos\alpha\sin\frac{\alpha}{2}-\sin\alpha\cos\frac{\alpha}{2})}{\|\bu_0\|} \right)
\frac{\left( \cos\alpha\cos\frac{\alpha}{2}+
  \sin\alpha\sin\frac{\alpha}{2}\right)}{\|\bu_0\|}\\
\overset{(0)}{\geq}&
  \frac{\alpha}{5 \|\bu_0\|^2\pi}
\sigma\left(\|\bu_0\|\frac{\sin(\pi-\frac{\alpha}{2})}{\|\bu_0\|}\right)
\sigma'\left(\|\bu_0\|\frac{-2\sin\frac{\alpha}{2}}{\|\bu_0\|} \right)
\left( \cos\frac{\alpha}{2}\right)\\
\overset{(1)}{=}&
  \frac{\alpha}{5 \|\bu_0\|^2\pi}
\sigma\left(\sin\frac{\alpha}{2}\right)
\sigma'\left(2\sin\frac{\alpha}{2}\right)
\left( \cos\frac{\alpha}{2}\right)\\
\geq&  \frac{\alpha}{5 \|\bu_0\|^2\pi}
\sigma\left(\sin\frac{\alpha}{2}\right)
\sigma'\left(2\sin\frac{\pi}{4}\right)
\left( \cos\frac{\pi}{4}\right)\\
:=&f(\alpha, \|\bu_0\|),
\end{align*}
where $(0)$ is from $\|\bw_0\|\leq \|\bu_0\|$ (the projection step we
might have added does not change the angle) and $(1)$ from symmetry of
$\sigma'$.  Observe that the $\erf$'s derivative at 0 is 1. Combined
with the fact that the sine function behaves the same at the neighbourhood of 0, we
obtain that
$\sigma\left(\sin\frac{\alpha}{2}\right)=\Theta(\frac{\alpha}{2})$,
when $\alpha\rightarrow 0$. We obtain that
$f(\alpha, \|\bu_0\|)=\Omega(\frac{\alpha^2}{\|\bu_0\|^2})$.  Now, we
use this in order to examine the angular improvement. Since
$\|\bw_0\|\leq \|\bu_0\|$, we obtain:
\[
\delta_\alpha^{(t+1)}:=\alpha^{(t)}-\alpha^{(t+1)} \geq \tan^{-1}
\left(\frac{f(\alpha^{(t)}, \|\bu_0\|)}{\|\bu_0\|}\right) ~.
\]
For reasonably large $k$, the argument of the arctan is smaller than
$1$, and on the interval $[0,1]$ the value of arctan is larger than
half its argument. This implies that $\delta_\alpha^{(t+1)} \ge
\Omega(\alpha^2/k^3)$, or
\[
\alpha^{(t+1)} \le \left(1 - \Omega((\alpha^{(t)})^2/k^3)\right) \,
\alpha^{(t)} ~.
\] 
This proves that after $O((k/\epsilon)^3)$ iterations the value of
$\alpha$ must be smaller than $\epsilon$. 
\end{proof}

\begin{proof}{\bf\label{proof:thm:isotron_gradient_correlation} of
  \thmref{thm:isotron_gradient_correlation}:}  Let $r=\|\bu_0\|$. Fix
  $\bw_0\in B^d(0,r)$. 
For some $s>0$, define the event
  $A_s=\{\bx_1~:~\max\{|\bu_0^\top\bx_1|, |\bw_0^\top\bx_1|,
  |(\bu_0-\bw_0)^\top\bx_1|\} < s\}$. Now,
\begin{align*}
\inner{\bw_0-\bu_0,\nabla
  F^{(1)}(\bw_0)}=&\E_{\bx_1}\left[  \left(
                  \sigma (\bw_0^\top\bx_1) - \sigma (\bu_0^\top\bx_1) 
                  \right) \sigma '(\bw_0^\top\bx_1)\cdot
   \inner{\bw_0-\bu_0,\bx_1}\right]
\end{align*}
Note that $\sigma' > 0$, hence $\sigma$ is monotonically increasing, 
implying non negativity of the inner expression for all $\bx_1$. For some $s$, let
$L(s)=\sigma '(s)$. Then, by the properties of $\sigma$, for
all $a\in[-s,s]$, $\sigma '(a)\geq L(s) > 0$. We proceed:
\begin{align*}
  \inner{\bw_0-\bu_0,\nabla F^{(1)}(\bw_0)}\overset{(1)}{\geq}&\E_{\bx_1}\left[ \left(
                                                            \sigma(\bw_0^\top\bx_1) - \sigma (\bu_0^\top\bx_1) 
                                                            \right) \sigma '(\bw_0^\top\bx_1)\cdot
                                                            \inner{\bw_0-\bu_0,\bx_1}\mathbf{1}_{A_s}\right]\\
  \overset{(2)}{\geq}&L(s)\cdot\E_{\bx_1}\left[\left(
                       \sigma (\bw_0^\top\bx_1) - \sigma (\bu_0^\top\bx_1) 
                       \right) \cdot
                       \inner{\bw_0-\bu_0,\bx_1}\mathbf{1}_{A_s}\right]\\
  \overset{(3)}{\geq}&L^2(s)\cdot\E_{\bx_1}\left[
                       \inner{\bw_0-\bu_0,\bx_1}^2 \mathbf{1}_{A_s}\right]
\end{align*}
where $(1)$ is from non negativity of the inner expression, $(2)$ from
the lower bound over the derivative of $\sigma$, applicable from the
occurence of $A_s$, $(3)$ from the Mean Value Theorem, again
applicable for the event $A_s$. $\bx_1$ is standard Gaussian, so
$\E_{\bx_1}\left[ \inner{\bw_0-\bu_0,\bx_1}^2\right] =
\|\bw_0-\bu_0\|_2^2$. 
We continue to develop the lower bound:
\begin{align*}
  \inner{\bw_0-\bu_0,\nabla F^{(1)}(\bw_0)}\geq&L^2(s)\cdot\E_{\bx_1}\left[
                                             \inner{\bw_0-\bu_0,\bx_0}^2
                                             \mathbf{1}_{A_s}\right]\\
  =&L^2(s)\cdot\left(\E_{\bx_1}\left[
     \inner{\bw_0-\bu_0,\bx_1}^2\right] -
     \E_{\bx_1}\left[\inner{\bw_0-\bu_0,\bx_1}^2\mathbf{1}_{A^c_s}\right]\right)
  \\
\overset{(1)}{=}&L^2(s)\cdot\left(\|\bw_0-\bu_0\|^2 -
     \E_{\bx_1}\left[\inner{\bw_0-\bu_0,\bx_1}^2\mathbf{1}_{A^c_s}\right]\right)\\
\overset{(2)}{\geq}&L^2(s)\cdot\left(\|\bw_0-\bu_0\|^2 -
     \left(\E_{\bx_1}\left[\inner{\bw_0-\bu_0,\bx_1}^4\right]\Pr(A_s^c)\right)^{1/2}\right)\\
\overset{(3)}{=}&L^2(s)\cdot\left(\|\bw_0-\bu_0\|^2 -
     \|\bw_0-\bu_0\|^2\left(3\Pr(A_s^c)\right)^{1/2}\right)\\
=&L^2(s)\cdot \|\bw_0-\bu_0\|^2\left(1 - \left(3\Pr(A_s^c)\right)^{1/2}\right)
\end{align*}
where $(1)$ is from the fact $\bx_1$ is standard Gaussian, $(2)$ is
from Cauchy-Schwartz, $(3)$ is from direct computation using the
fourth moment of a Gaussian. Let $\tilde{s}$ be sufficiently large,
such that $\left(3\Pr(A_{\tilde{s}}^c)\right)^{1/2} < 1-1/k$. This is
possible for $s=\Theta(1)$, since, when assuming w.l.o.g.
$\bu_0\in\reals^2$,
$\Pr(\bu_0^\top\bx<1) \geq \frac{1}{\|\bu_0\|}\frac{1}{\sqrt{2\pi}}
e^{\frac{-1}{2\|\bu_0\|^2}}=\Omega\left(k^{-1}\right)$, and the same
for $\bw_0,\bu_0-\bw_0$. Then,
\[
  \inner{\bw_0-\bu_0,\nabla F_{\bu_0}(\bw_0)}~\geq~\frac{L^2 (\tilde{s})}{k}\cdot \|\bw_0-\bu_0\|^2
\]

\end{proof}

\subsubsection{Proof of \thmref{thm:Fsum_gradient_correlation}}
Before proving \thmref{thm:Fsum_gradient_correlation}, we state and prove two
technical lemmas:
\begin{lemma}\label{lem:negative_corr}
  Let $\alpha_1=\tan^{-1}\frac{\epsilon}{2\|\bu_0\|}$, and assume
  $\alpha\in[0,\alpha_1)$. Then at least one of the following holds:
\begin{itemize}
\item   $\|\bu_0-\bw_0\|<\epsilon$,
\item   $ \sign\left(\inner{\bw_0-\bu_0,\bw_0}\right)=-1$.
\end{itemize}
\end{lemma}
\begin{proof}
It is easy to verify that if the first condition doesn't hold, then
the worst case scenario is when $\alpha = \alpha_1$ and $\bw_0$ lies
on the $\epsilon$-sphere around $\bu_0$. So, from now on we assume
that $\alpha = \alpha_1$ and $\bw_0$ lies
on the $\epsilon$-sphere around $\bu_0$. Consider the vector $\bz$,
whose angle w.r.t. $\bu_0$ is $\alpha$, and such that the angle
between $\bz$ and $\bu_0 - \bz$ is exactly $90$ degrees. Since $\bz$
and $\bw_0$ points to the same direction, It is easy to 
verify that it suffices to show that $\|\bz-\bu_0\| < \epsilon$. To
see this, we observe that $\frac{\|\bz\|}{\|\bu_0\|} =
\cos(\alpha)$, hence
\[
\|\bz\| = \|\bu_0\| \, \cos \tan^{-1} \frac{\epsilon}{2\|\bu_0\|} =
\|\bu_0\| \, \frac{1}{\sqrt{1 + \left( \frac{\epsilon}{2\|\bu_0\|}\right)^2}}
\]
By the Pythagorean theorem,
\[
\|\bz-\bu_0\|^2 = \|\bu_0\|^2 - \|\bz\|^2= 
\|\bu_0\|^2 \,\left(1 - \frac{1}{1 + \left(
      \frac{\epsilon}{2\|\bu_0\|}\right)^2} \right) \le \frac{\epsilon^2}{4} ~,
\] 
which concludes our proof.
 \end{proof}
\begin{lemma}\label{lem:good_cos}
  Assume
  $0\leq\alpha<\min\{\tan^{-1}\left(\frac{\epsilon}{2\|\bu_0\|}\right),\sqrt{\frac{\epsilon}{\|\bu_0\|}}\}$. Then
  at least one of the following holds:
\begin{itemize}
\item   $\|\bu_0-\bw_0\|<\epsilon$,
\item $\frac{\|\bw_0\|
  \sqrt{1+2\|\bu_0\|^2}}{\sqrt{1+2\|\bw_0\|^2}\|\bu_0\|}\leq\cos\alpha$.
\end{itemize}
\end{lemma}
\begin{proof}
  It is clear that if not $\|\bu_0-\bw_0\|<\epsilon$, then
  $\|\bw_0\|<\|\bu_0\|-\frac{\epsilon}{2}$, by the fact
  $\alpha<\tan^{-1}\left(\frac{\epsilon}{2\|\bu_0\|}\right)$. Define
  the function $ f(x)=\frac{x}{\sqrt{1+2x^2}} $. We compute its
  derivatives, and note inequalities for $x>1$:
\[
  f'(x)=\frac{1+\frac{2x^2}{1+2x^2}}{\sqrt{1+2x^2}}>0,~~
  f''(x)=\frac{\frac{4x^2}{1+2x^2}-\left(1+\frac{2x^2}{1+2x^2}\right)2x}{\sqrt{1+2x^2}}<0.
\]
Observe that $f$ is monotonically increasing, and concave for $x>1$,
and in particular, for $x=\|\bu_0\|>1$. Therefore,
$f(\|\bw_0\|)<f(\|\bu_0\|-\epsilon/2)<f(\|\bu_0\|) -\frac{\epsilon\cdot
  f'(\|\bu_0\|)}{2}$. We obtain:
\begin{align*}
  \frac{\|\bw_0\|
  \sqrt{1+2\|\bu_0\|^2}}{\sqrt{1+2\|\bw_0\|^2}\|\bu_0\|}~=&~\frac{f(\|\bw_0\|)}{f(\|\bu_0\|)}~
  \leq~\frac{f(\|\bu_0\|)
  -\frac{\epsilon\cdot
  f'(\|\bu_0\|)}{2}}{f(\|\bu_0\|)}~=~1-\frac{\epsilon}{2}\frac{f'(\|\bu_0\|)}{f(\|\bu_0\|)}\\
  ~=&~1-\frac{\epsilon}{2}\left(\frac{1}{\|\bu_0\|}+\frac{2 \|\bu_0\|}{1+2
  \|\bu_0\|^2}\right) 
  ~\leq~1-\frac{\epsilon}{2 \|\bu_0\|}
\end{align*}
From observing the Taylor series of $\cos$, we have that
$\cos\alpha\geq1-\frac{\alpha^2}{2}$. Thus, from the assumption
$0\leq\alpha<\sqrt{\frac{\epsilon}{\|\bu_0\|}}$,
\[
\cos\alpha\geq\cos\sqrt{\frac{\epsilon}{\|\bu_0\|}}\geq1-\frac{\epsilon}{2\|\bu_0\|}\geq \frac{\|\bw_0\|
  \sqrt{1+2\|\bu_0\|^2}}{\sqrt{1+2\|\bw_0\|^2}\|\bu_0\|}.
\]
\end{proof}

\begin{proof}{\bf \label{proof:thm:Fsum_gradient_correlation} of
  \thmref{thm:Fsum_gradient_correlation}.} $\nabla F^{(1,k)}$ can be written as:
\begin{align*}
\nabla F^{(1,k)}(\bw_0) = &~\nabla F^{(1)}(\bw_0)+\nabla F^{(k)}(\bw_0)\\
=&~c_1(\bw_0)\bw_0-\tilde{\bb}_1+kV_{\sigma}(\bw_0,\bw_0)^{k-1}\cdot
   c_1(\bw_0)\bw_0-kV_\sigma(\bu_0,\bw_0)^{k-1}\cdot \tilde{\bb}_1\\
=&~(1+kV_{\sigma}(\bw_0,\bw_0)^{k-1})c_1(\bw_0)\bw_0 -
   (1+kV_\sigma(\bu_0,\bw_0)^{k-1}) \tilde{\bb}_1\\
=&~\left((1+kV_{\sigma}(\bw_0,\bw_0)^{k-1})-(1+kV_\sigma(\bu_0,\bw_0)^{k-1})
   \right)c_1(\bw_0)\bw_0 +(1+kV_\sigma(\bu_0,\bw_0)^{k-1}) \nabla F^{(1)}\\
=&~kc_1(\bw_0)\left(V_{\sigma}(\bw_0,\bw_0)^{k-1}-V_\sigma(\bu_0,\bw_0)^{k-1} \right)\bw_0 +(1+kV_\sigma(\bu_0,\bw_0)^{k-1}) \nabla F^{(1)}.
\end{align*}
We already have, by \thmref{thm:isotron_gradient_correlation}, that
$ \inner{\bw_0-\bu_0,\nabla F^{(1)}}\geq \frac{L^2(\tilde{s})}{k}\|\bw_0-\bu_0\|^2$, and
therefore (as $\alpha<\pi/2$, implying
$(1+kV_\sigma(\bu_0,\bw_0)^{k-1})>1$), a similar result applies for this
term of the gradient. It is hence sufficient to show that the first
term does not worsen the lower bound, namely,
\begin{align*}
\sign\left(\inner{\bw_0-\bu_0,
  kc_1(\bw_0)\left(V_{\sigma}(\bw_0,\bw_0)^{k-1}-V_\sigma(\bu_0,\bw_0)^{k-1}
  \right)\bw_0}\right)\geq 0
\end{align*}
We observe a few equivalencies:
\begin{align*}
&~\sign\left(\inner{\bw_0-\bu_0,
  kc_1(\bw_0)\left(V_{\sigma}(\bw_0,\bw_0)^{k-1}-V_\sigma(\bu_0,\bw_0)^{k-1}
  \right)\bw_0}\right)\\
\overset{(0)}{=}&~\sign\left(\inner{\bw_0-\bu_0,
  \left(V_{\sigma}(\bw_0,\bw_0)^{k-1}-V_\sigma(\bu_0,\bw_0)^{k-1}
  \right)\bw_0}\right)\\
=&~\sign\left(V_{\sigma}(\bw_0,\bw_0)^{k-1}-V_\sigma(\bu_0,\bw_0)^{k-1}\right)\cdot\sign\left(\inner{\bw_0-\bu_0,\bw_0}\right)\\
\overset{(1)}{=}&~-\sign\left(V_{\sigma}(\bw_0,\bw_0)^{k-1}-V_\sigma(\bu_0,\bw_0)^{k-1}\right)\\
\overset{(2)}{=}&~-\sign\left(\frac{2\bw_0^\top\bw_0}{\sqrt{1+2\|\bw_0\|^2}\sqrt{1+2\|\bw_0\|^2}}-\frac{2\bw_0^\top\bu_0}{\sqrt{1+2\|\bw_0\|^2}\sqrt{1+2\|\bu_0\|^2}}\right)\\
=&~-\sign\left(\frac{\bw_0^\top\bw_0}{\sqrt{1+2\|\bw_0\|^2}}-\frac{\bw_0^\top\bu_0}{\sqrt{1+2\|\bu_0\|^2}}\right)\\
\end{align*}
where $(0)$ is from the fact $kc_1(\bw_0)\geq 0$, $(1)$ is from
\lemref{lem:negative_corr}, and $(2)$ is from positivity of
$\bw_0^\top\bu_0$, and monotonicity of the $\sin^{-1}$ in the explicit
expression for $V_{\erf}$, for positive $\bw_0^\top\bu_0$. It is left
to show
$ \frac{\bw_0^\top\bw_0
  \sqrt{1+2\|\bu_0\|^2}}{\sqrt{1+2\|\bw_0\|^2}}\leq
\bw_0^\top\bu_0=\|\bw_0\|\|\bu_0\|\cos\alpha$, equivalent, after
rearrangement, to
$ \frac{\|\bw_0\|
  \sqrt{1+2\|\bu_0\|^2}}{\sqrt{1+2\|\bw_0\|^2}\|\bu_0\|}\leq
\cos\alpha $, which we have from \lemref{lem:good_cos}.
\end{proof}

\subsubsection{Proof of \thmref{thm:phase2}}
\begin{proof}{\bf \label{proof:thm:phase2} of \thmref{thm:phase2}.}
Combining \thmref{thm:Fsum_gradient_correlation} with the Cauchy-Schwartz inequality we obtain that,
\[
\|\nabla F^{(1,k)}(\bw_0)\|~\geq~\frac{L^2(\tilde{s})}{k}\cdot \|\bw_0-\bu_0\|.
\]
On the other hand, it is easy to see that
$\|\nabla F^{(1,k)}(\bw_0)\| \le O(k^2)$.
Let $\bw_0^{(t)}$ be the weight vector after the $t$th iteration of
GD, $\eta$ the learning rate. Assume that for all $t'<t$, we did not
converge yet, namely, $\|\bw_0^{(t')}-\bu_0\| > \epsilon$. Then,
\begin{align*}
  \|\bw_0^{(t+1)}-\bu_0\|^2 - \|\bw_0^{(t)}-\bu_0\|^2=&
                                                    \|\bw_0^{(t)}-\eta \nabla F(\bw_0^{(t)})-\bu_0\|^2 -
                                                    \|\bw_0^{(t)}-\bu_0\|^2\\
  =&-2\eta \inner{\nabla F(\bw_0^{(t)}), \bw_0^{(t)}-\bu_0} + \eta^2\|\nabla F(\bw_0^{(t)})\|^2
\end{align*}
We can now use the previous bounds from
\thmref{thm:Fsum_gradient_correlation} to obtain, after rearranging:
\begin{align*}
  \|\bw_0^{(t+1)}-\bu_0\|^2 =&
                                     \|\bw_0^{(t)}-\bu_0\|^2-2\eta \inner{\nabla F(\bw_0^{(t)}), \bw_0^{(t)}-\bu_0} + \eta^2\|\nabla F(\bw_0^{(t)})\|^2\\
\leq&\|\bw_0^{(t)}-\bu_0\|^2 -2\eta\frac{ L^2(\tilde{s})}{k}\cdot \|\bw_0-\bu_0\|^2+
      \eta^2\,k^4
\end{align*}
Taking $\eta = \frac{L^2(\tilde{s})\,\epsilon^2}{k^5}$ we obtain that
as long as 
$\|\bw_0^{(t)}-\bu_0\|>\epsilon$ then 
\[
  \|\bw_0^{(t+1)}-\bu_0\|^2 \leq \|\bw_0^{(t)}-\bu_0\|^2 -
  \frac{L^4(\tilde{s})\,\epsilon^4}{k^6} ~,
\]
which implies that after at most $\poly(1/\epsilon,k)$ iterations we
must have $\|\bw_0^{(t)}-\bu_0\| \le \epsilon$.
\end{proof}

\end{document}

%% file: figs/unknown_g_5_mnist_mytikz.tex
\begin{tikzpicture}

\begin{axis}[
xmin=0, xmax=3000,
ymin=0, ymax=0.6,
axis on top,
width=\figurewidth,
height=\figureheight,
xtick={0,3000},
xticklabels={$0$,$3\cdot10^{3}$},
ytick={0,0.3,0.6},
yticklabels={$0$,$0.3$,$0.6$}
]
\addplot [very thick, red, dashed]
table {%
0 0.163890287280083
100 0.140997007489204
200 0.132108196616173
300 0.137684047222137
400 0.138258218765259
500 0.134883254766464
600 0.13537260890007
700 0.135500743985176
800 0.137576073408127
900 0.133296921849251
1000 0.127762168645859
1100 0.136014029383659
1200 0.132040292024612
1300 0.139097601175308
1400 0.138666868209839
1500 0.126968830823898
1600 0.13395631313324
1700 0.139451578259468
1800 0.139353066682816
1900 0.13821767270565
2000 0.142542332410812
2100 0.137134000658989
2200 0.135872527956963
2300 0.139025539159775
2400 0.134931415319443
2500 0.134148105978966
2600 0.138452649116516
2700 0.137023836374283
2800 0.136434227228165
2900 0.144163981080055
3000 0.145933866500854
};
\addplot [very thick, blue]
table {%
0 0.209830835461617
100 0.187348201870918
200 0.178465470671654
300 0.180739656090736
400 0.187516361474991
500 0.181362375617027
600 0.186449214816093
700 0.188280418515205
800 0.181747049093246
900 0.185080215334892
1000 0.184443086385727
1100 0.176746413111687
1200 0.181931838393211
1300 0.180479437112808
1400 0.19208499789238
1500 0.186812624335289
1600 0.186503186821938
1700 0.186124846339226
1800 0.185269355773926
1900 0.181413486599922
2000 0.181693881750107
2100 0.192160576581955
2200 0.183374404907227
2300 0.190210893750191
2400 0.183568060398102
2500 0.185326084494591
2600 0.183485001325607
2700 0.177543684840202
2800 0.186663523316383
2900 0.186718985438347
3000 0.19341553747654
};
\end{axis}

\end{tikzpicture}

%% file: figs/unknown_g_1_mnist_mytikz.tex
\begin{tikzpicture}

\begin{axis}[
xmin=0, xmax=3000,
ymin=0, ymax=0.6,
axis on top,
width=\figurewidth,
height=\figureheight,
xtick={0,3000},
xticklabels={$0$,$3\cdot10^{3}$},
ytick={0,0.3,0.6},
yticklabels={$0$,$0.3$,$0.6$}
]
\addplot [very thick, red, dashed]
table {%
0 0.422977536916733
100 0.0992296487092972
200 0.0611716322600842
300 0.053149439394474
400 0.0480308197438717
500 0.0406188070774078
600 0.0390791036188602
700 0.0357452109456062
800 0.032189853489399
900 0.0284630302339792
1000 0.0273695345968008
1100 0.0268633812665939
1200 0.024285938590765
1300 0.0229397565126419
1400 0.0213486012071371
1500 0.0209593661129475
1600 0.018781453371048
1700 0.0187173020094633
1800 0.0182426124811172
1900 0.0168258231133223
2000 0.0177701897919178
2100 0.0141947967931628
2200 0.0147676691412926
2300 0.0142638832330704
2400 0.0173650495707989
2500 0.013708489947021
2600 0.0137313157320023
2700 0.0136502115055919
2800 0.0121818603947759
2900 0.012010732665658
3000 0.0104699870571494
};
\addplot [very thick, blue]
table {%
0 0.450236618518829
100 0.0684396177530289
200 0.0466163568198681
300 0.0377337522804737
400 0.0316617116332054
500 0.0291351042687893
600 0.0223786421120167
700 0.0201085470616817
800 0.018775375559926
900 0.0191847942769527
1000 0.0155177116394043
1100 0.0152136506512761
1200 0.0128591954708099
1300 0.0145348515361547
1400 0.0121677685528994
1500 0.0102440193295479
1600 0.00936947949230671
1700 0.00868151616305113
1800 0.00852944701910019
1900 0.0080998046323657
2000 0.00698551395907998
2100 0.00668927747756243
2200 0.00644180784001946
2300 0.00576229207217693
2400 0.00554093485698104
2500 0.00558255799114704
2600 0.00454070046544075
2700 0.00537234032526612
2800 0.00423320662230253
2900 0.00400273315608501
3000 0.00408791145309806
};
\end{axis}

\end{tikzpicture}

%% file: figs/unknown_g_1_5_mnist_mytikz.tex
\begin{tikzpicture}

\begin{axis}[
xmin=0, xmax=3000,
ymin=0, ymax=0.6,
axis on top,
width=\figurewidth,
height=\figureheight,
xtick={0,3000},
xticklabels={$0$,$3\cdot10^{3}$},
ytick={0,0.3,0.6},
yticklabels={$0$,$0.3$,$0.6$}
]
\addplot [very thick, red, dashed]
table {%
0 0.555326759815216
100 0.275717318058014
200 0.242361307144165
300 0.225302889943123
400 0.207543581724167
500 0.201110601425171
600 0.209561944007874
700 0.197150513529778
800 0.189435407519341
900 0.1961590051651
1000 0.18341138958931
1100 0.178798153996468
1200 0.185314446687698
1300 0.174957126379013
1400 0.1771320104599
1500 0.167125374078751
1600 0.159935653209686
1700 0.163594201207161
1800 0.155357405543327
1900 0.151252716779709
2000 0.141874238848686
2100 0.132434278726578
2200 0.144507020711899
2300 0.124974153935909
2400 0.111904643476009
2500 0.106147930026054
2600 0.0995422005653381
2700 0.0930772498250008
2800 0.0821916610002518
2900 0.0807057321071625
3000 0.0690231844782829
};
\addplot [very thick, blue]
table {%
0 0.660789370536804
100 0.307020574808121
200 0.270042657852173
300 0.274105370044708
400 0.248621851205826
500 0.259499490261078
600 0.242131948471069
700 0.242086797952652
800 0.242938801646233
900 0.240284651517868
1000 0.236302137374878
1100 0.235875129699707
1200 0.227361708879471
1300 0.237522318959236
1400 0.23359639942646
1500 0.230248242616653
1600 0.224646091461182
1700 0.22751846909523
1800 0.227541640400887
1900 0.228261798620224
2000 0.232015937566757
2100 0.218185424804688
2200 0.218720555305481
2300 0.228774785995483
2400 0.223484471440315
2500 0.225715830922127
2600 0.226182773709297
2700 0.225686743855476
2800 0.226708367466927
2900 0.226578950881958
3000 0.221484422683716
};
\end{axis}

\end{tikzpicture}

%% file: figs/unknown_g_5_mytikz.tex
\begin{tikzpicture}

\begin{axis}[
xmin=0, xmax=3000,
ymin=0, ymax=0.6,
axis on top,
width=\figurewidth,
height=\figureheight,
xtick={0,3000},
xticklabels={$0$,$3\cdot10^{3}$},
ytick={0,0.3,0.6},
yticklabels={$0$,$0.3$,$0.6$}
]
\addplot [very thick, red, dashed]
table {%
0 0.49012479186058
100 0.334388494491577
200 0.327573955059052
300 0.32730421423912
400 0.335926622152328
500 0.32831221818924
600 0.339788019657135
700 0.325623959302902
800 0.329391151666641
900 0.337845653295517
1000 0.33883148431778
1100 0.333121120929718
1200 0.337045520544052
1300 0.332112520933151
1400 0.327415615320206
1500 0.321569621562958
1600 0.332980066537857
1700 0.326670259237289
1800 0.322460412979126
1900 0.327634483575821
2000 0.343096643686295
2100 0.327415883541107
2200 0.33059960603714
2300 0.330636441707611
2400 0.355952799320221
2500 0.336628645658493
2600 0.337665617465973
2700 0.328736126422882
2800 0.336202532052994
2900 0.330984890460968
3000 0.337737619876862
};
\addplot [very thick, blue]
table {%
0 0.534569084644318
100 0.333698838949203
200 0.342666953802109
300 0.338693201541901
400 0.334885746240616
500 0.339073330163956
600 0.344387054443359
700 0.334612101316452
800 0.333813726902008
900 0.33802056312561
1000 0.335135042667389
1100 0.329219043254852
1200 0.329973191022873
1300 0.327884942293167
1400 0.325576364994049
1500 0.333480060100555
1600 0.33432549238205
1700 0.331291317939758
1800 0.331352710723877
1900 0.327600687742233
2000 0.324138075113297
2100 0.333999454975128
2200 0.338751971721649
2300 0.336101114749908
2400 0.331476390361786
2500 0.32696670293808
2600 0.33381724357605
2700 0.32881298661232
2800 0.337412297725677
2900 0.326602727174759
3000 0.325155556201935
};
\end{axis}

\end{tikzpicture}

%% file: figs/unknown_g_1_mytikz.tex
\begin{tikzpicture}

\begin{axis}[
xmin=0, xmax=3000,
ymin=0, ymax=0.6,
axis on top,
width=\figurewidth,
height=\figureheight,
xtick={0,3000},
xticklabels={$0$,$3\cdot10^{3}$},
ytick={0,0.3,0.6},
yticklabels={$0$,$0.3$,$0.6$}
]
\addplot [very thick, red, dashed]
table {%
0 0.485293209552765
100 0.460439741611481
200 0.458221107721329
300 0.460943073034286
400 0.45804300904274
500 0.464145839214325
600 0.454156696796417
700 0.457530677318573
800 0.463050782680511
900 0.459918886423111
1000 0.462389230728149
1100 0.458269596099854
1200 0.462479591369629
1300 0.462425142526627
1400 0.457634389400482
1500 0.463636785745621
1600 0.461992740631104
1700 0.457748502492905
1800 0.326263606548309
1900 0.0113485716283321
2000 0.00501001067459583
2100 0.00384192960336804
2200 0.00277521763928235
2300 0.00270823412574828
2400 0.00214589410461485
2500 0.00173574278596789
2600 0.00174679979681969
2700 0.00187036022543907
2800 0.00165076518896967
2900 0.00126696191728115
3000 0.00143205362837762
};
\addplot [very thick, blue]
table {%
0 0.596107959747314
100 0.0285299066454172
200 0.0154743809252977
300 0.0100398380309343
400 0.0064354739151895
500 0.00486478442326188
600 0.00731625268235803
700 0.00287576788105071
800 0.0027921968139708
900 0.0021889777854085
1000 0.00189237203449011
1100 0.00183289533015341
1200 0.0015528139192611
1300 0.00128924835007638
1400 0.00146536459214985
1500 0.00120803015306592
1600 0.00133143388666213
1700 0.000933555245865136
1800 0.00089995824964717
1900 0.000816433865111321
2000 0.000659383076708764
2100 0.000806728086899966
2200 0.000701118318829685
2300 0.000683704798575491
2400 0.000735914043616503
2500 0.000605760258622468
2600 0.000511686317622662
2700 0.000526069430634379
2800 0.000493077852297574
2900 0.00047880393685773
3000 0.000713342858944088
};
\end{axis}

\end{tikzpicture}

%% file: figs/unknown_g_1_5_mytikz.tex
\begin{tikzpicture}

\begin{axis}[
xmin=0, xmax=3000,
ymin=0, ymax=0.6,
axis on top,
width=\figurewidth,
height=\figureheight,
xtick={0,3000},
xticklabels={$0$,$3\cdot10^{3}$},
ytick={0,0.3,0.6},
yticklabels={$0$,$0.3$,$0.6$}
]
\addplot [very thick, red, dashed]
table {%
0 0.881084263324738
100 0.804078876972198
200 0.736261069774628
300 0.543495535850525
400 0.239438056945801
500 0.100169226527214
600 0.0469452030956745
700 0.0247317850589752
800 0.022945549339056
900 0.0124440882354975
1000 0.0296362079679966
1100 0.00662112375721335
1200 0.0160963945090771
1300 0.00452129263430834
1400 0.00583856226876378
1500 0.00433611962944269
1600 0.00530788209289312
1700 0.00440473994240165
1800 0.0162374339997768
1900 0.00340659031644464
2000 0.00513938488438725
2100 0.00285411649383605
2200 0.00281036784872413
2300 0.00253974506631494
2400 0.00322756543755531
2500 0.0027195259463042
2600 0.00230872211977839
2700 0.00239171413704753
2800 0.00191593286581337
2900 0.00212339544668794
3000 0.00478979153558612
};
\addplot [very thick, blue]
table {%
0 1.20563733577728
100 0.410785436630249
200 0.364909708499908
300 0.382727473974228
400 0.355535119771957
500 0.34881266951561
600 0.341424465179443
700 0.350256472826004
800 0.353160798549652
900 0.345229417085648
1000 0.335686475038528
1100 0.341042816638947
1200 0.379557281732559
1300 0.347100406885147
1400 0.341458559036255
1500 0.338790327310562
1600 0.331727236509323
1700 0.332871794700623
1800 0.340398848056793
1900 0.337401568889618
2000 0.333374321460724
2100 0.351642608642578
2200 0.333506464958191
2300 0.336948692798615
2400 0.343955814838409
2500 0.33393207192421
2600 0.353067308664322
2700 0.348245680332184
2800 0.336408764123917
2900 0.345806121826172
3000 0.332703411579132
};
\end{axis}

\end{tikzpicture}

%% file: figs/fix_g_5_mytikz.tex
\begin{tikzpicture}

\begin{axis}[
xmin=0, xmax=3000,
ymin=0, ymax=0.6,
axis on top,
width=\figurewidth,
height=\figureheight,
xtick={0,3000},
xticklabels={$0$,$3\cdot10^{3}$},
ytick={0,0.3,0.6},
yticklabels={$0$,$0.3$,$0.6$}
]
\addplot [very thick, red, dashed]
table {%
0 0.34108430147171
100 0.332397639751434
200 0.331526219844818
300 0.339878380298615
400 0.333769500255585
500 0.339643716812134
600 0.337745249271393
700 0.327856838703156
800 0.338384658098221
900 0.32990351319313
1000 0.322028696537018
1100 0.333542048931122
1200 0.329436212778091
1300 0.329958528280258
1400 0.321977198123932
1500 0.333202660083771
1600 0.33237013220787
1700 0.338528960943222
1800 0.334778368473053
1900 0.328781604766846
2000 0.329011827707291
2100 0.333288729190826
2200 0.332617282867432
2300 0.334158360958099
2400 0.338065594434738
2500 0.333085238933563
2600 0.329150259494781
2700 0.33281621336937
2800 0.327179610729218
2900 0.327595770359039
3000 0.323806464672089
};
\addplot [very thick, blue]
table {%
0 0.328397691249847
100 0.333525776863098
200 0.332438200712204
300 0.325191795825958
400 0.328982532024384
500 0.327681988477707
600 0.325006783008575
700 0.328020662069321
800 0.333445310592651
900 0.335384517908096
1000 0.333728820085526
1100 0.327894628047943
1200 0.335239976644516
1300 0.334553956985474
1400 0.335889995098114
1500 0.325608164072037
1600 0.328861892223358
1700 0.327311813831329
1800 0.32490748167038
1900 0.330283135175705
2000 0.319538086652756
2100 0.330730199813843
2200 0.322676122188568
2300 0.330676794052124
2400 0.332267224788666
2500 0.326907336711884
2600 0.334466695785522
2700 0.327623933553696
2800 0.333122193813324
2900 0.325089752674103
3000 0.337453275918961
};
\end{axis}

\end{tikzpicture}

%% file: figs/fix_g_1_mytikz.tex
\begin{tikzpicture}

\begin{axis}[
xmin=0, xmax=3000,
ymin=0, ymax=0.6,
axis on top,
width=\figurewidth,
height=\figureheight,
xtick={0,3000},
xticklabels={$0$,$3\cdot10^{3}$},
ytick={0,0.3,0.6},
yticklabels={$0$,$0.3$,$0.6$}
]
\addplot [very thick, red, dashed]
table {%
0 0.49179419875145
100 0.027592858299613
200 0.0189552493393421
300 0.0113962143659592
400 0.00958146620541811
500 0.00763113563880324
600 0.00708655780181289
700 0.00655610393732786
800 0.00604575406759977
900 0.00498165842145681
1000 0.00463919574394822
1100 0.00400416040793061
1200 0.0035711124073714
1300 0.00343530206009746
1400 0.0034831881057471
1500 0.00323676923289895
1600 0.00260252458974719
1700 0.00221733888611197
1800 0.00217238184995949
1900 0.00250035035423934
2000 0.00222531403414905
2100 0.00200414331629872
2200 0.00192758231423795
2300 0.00181872176472098
2400 0.00159778515808284
2500 0.00167556921951473
2600 0.00161823409143835
2700 0.00145775242708623
2800 0.00151233770884573
2900 0.00135947368107736
3000 0.0011524970177561
};
\addplot [very thick, blue]
table {%
0 0.440343469381332
100 0.0274302419275045
200 0.0168289784342051
300 0.0135007444769144
400 0.00958266574889421
500 0.00681430380791426
600 0.00730806589126587
700 0.00646678265184164
800 0.00547632807865739
900 0.00533617381006479
1000 0.00360779464244843
1100 0.00419429829344153
1200 0.00359313632361591
1300 0.00295203225687146
1400 0.0029548448510468
1500 0.00283644208684564
1600 0.00279692583717406
1700 0.00269894883967936
1800 0.00250361463986337
1900 0.00208973907865584
2000 0.00198574503883719
2100 0.00178882631007582
2200 0.00178720033727586
2300 0.0017903100233525
2400 0.00177296786569059
2500 0.00138395337853581
2600 0.00153716653585434
2700 0.00128513411618769
2800 0.00135739822871983
2900 0.00122012943029404
3000 0.00123644445557147
};
\end{axis}

\end{tikzpicture}

%% file: figs/fix_g_1_5_mytikz.tex
\begin{tikzpicture}

\begin{axis}[
xmin=0, xmax=3000,
ymin=0, ymax=0.6,
axis on top,
width=\figurewidth,
height=\figureheight,
xtick={0,3000},
xticklabels={$0$,$3\cdot10^{3}$},
ytick={0,0.3,0.6},
yticklabels={$0$,$0.3$,$0.6$}
]
\addplot [very thick, red, dashed]
table {%
0 0.894790053367615
100 0.053489338606596
200 0.026479059830308
300 0.0174253135919571
400 0.0119714112952352
500 0.00801666174083948
600 0.00591589137911797
700 0.00513679021969438
800 0.00374392000958323
900 0.00311391032300889
1000 0.00257493602111936
1100 0.00211127707734704
1200 0.00198529381304979
1300 0.00160853413399309
1400 0.00141744967550039
1500 0.00104580190964043
1600 0.000998780713416636
1700 0.000915936368983239
1800 0.000671344285365194
1900 0.000644461601041257
2000 0.000547777570318431
2100 0.000471729465061799
2200 0.000395262526581064
2300 0.000372117472579703
2400 0.000357429962605238
2500 0.000279703468549997
2600 0.000287020899122581
2700 0.000234236620599404
2800 0.000208824261790141
2900 0.000180668401299044
3000 0.000154301451402716
};
\addplot [very thick, blue]
table {%
0 0.880793750286102
100 0.349622875452042
200 0.359203636646271
300 0.337923020124435
400 0.341906249523163
500 0.356673002243042
600 0.346240699291229
700 0.337417781352997
800 0.339635491371155
900 0.328880310058594
1000 0.333656847476959
1100 0.342444211244583
1200 0.336567670106888
1300 0.334935277700424
1400 0.339958727359772
1500 0.334870457649231
1600 0.329287976026535
1700 0.328903049230576
1800 0.339134126901627
1900 0.33783957362175
2000 0.339256823062897
2100 0.327564656734467
2200 0.343047231435776
2300 0.336561977863312
2400 0.335274636745453
2500 0.338020652532578
2600 0.333708941936493
2700 0.32615002989769
2800 0.327125608921051
2900 0.329530090093613
3000 0.328703552484512
};
\end{axis}

\end{tikzpicture}

%% file: paper_nips.bbl
\begin{thebibliography}{10}

\bibitem{andoni2014learning}
Alexandr Andoni, Rina Panigrahy, Gregory Valiant, and Li~Zhang.
\newblock Learning polynomials with neural networks.
\newblock In {\em International Conference on Machine Learning}, pages
  1908--1916, 2014.

\bibitem{arora2014provable}
Sanjeev Arora, Aditya Bhaskara, Rong Ge, and Tengyu Ma.
\newblock Provable bounds for learning some deep representations.
\newblock In {\em ICML}, pages 584--592, 2014.

\bibitem{blum1994weakly}
Avrim Blum, Merrick Furst, Jeffrey Jackson, Michael Kearns, Yishay Mansour, and
  Steven Rudich.
\newblock Weakly learning dnf and characterizing statistical query learning
  using fourier analysis.
\newblock In {\em Proceedings of the twenty-sixth annual ACM symposium on
  Theory of computing}, pages 253--262. ACM, 1994.

\bibitem{brutzkus2017globally}
Alon Brutzkus and Amir Globerson.
\newblock Globally optimal gradient descent for a convnet with gaussian inputs.
\newblock {\em arXiv preprint arXiv:1702.07966}, 2017.

\bibitem{choromanska2015loss}
Anna Choromanska, Mikael Henaff, Michael Mathieu, G{\'e}rard~Ben Arous, and
  Yann LeCun.
\newblock The loss surfaces of multilayer networks.
\newblock In {\em AISTATS}, 2015.

\bibitem{dachman2015approximate}
Dana Dachman-Soled, Vitaly Feldman, Li-Yang Tan, Andrew Wan, and Karl Wimmer.
\newblock Approximate resilience, monotonicity, and the complexity of agnostic
  learning.
\newblock In {\em Proceedings of the Twenty-Sixth Annual ACM-SIAM Symposium on
  Discrete Algorithms}, pages 498--511. Society for Industrial and Applied
  Mathematics, 2015.

\bibitem{daniely2017sgd}
Amit Daniely.
\newblock Sgd learns the conjugate kernel class of the network.
\newblock {\em arXiv preprint arXiv:1702.08503}, 2017.

\bibitem{haeffele2015global}
Benjamin~D Haeffele and Ren{\'e} Vidal.
\newblock Global optimality in tensor factorization, deep learning, and beyond.
\newblock {\em arXiv preprint arXiv:1506.07540}, 2015.

\bibitem{hardt2016identity}
Moritz Hardt and Tengyu Ma.
\newblock Identity matters in deep learning.
\newblock {\em arXiv preprint arXiv:1611.04231}, 2016.

\bibitem{janzamin2015beating}
Majid Janzamin, Hanie Sedghi, and Anima Anandkumar.
\newblock Beating the perils of non-convexity: Guaranteed training of neural
  networks using tensor methods.
\newblock {\em arXiv preprint arXiv:1506.08473}, 2015.

\bibitem{kakade2011efficient}
Sham~M Kakade, Varun Kanade, Ohad Shamir, and Adam Kalai.
\newblock Efficient learning of generalized linear and single index models with
  isotonic regression.
\newblock In {\em Advances in Neural Information Processing Systems}, pages
  927--935, 2011.

\bibitem{kalai2009isotron}
Adam~Tauman Kalai and Ravi Sastry.
\newblock The isotron algorithm: High-dimensional isotonic regression.
\newblock In {\em COLT}, 2009.

\bibitem{livni2014computational}
Roi Livni, Shai Shalev-Shwartz, and Ohad Shamir.
\newblock On the computational efficiency of training neural networks.
\newblock In {\em Advances in Neural Information Processing Systems}, pages
  855--863, 2014.

\bibitem{mei2016landscape}
Song Mei, Yu~Bai, and Andrea Montanari.
\newblock The landscape of empirical risk for non-convex losses.
\newblock {\em arXiv preprint arXiv:1607.06534}, 2016.

\bibitem{nesterov2004introductory}
Yurii Nesterov.
\newblock {\em Introductory lectures on convex optimization: A basic course},
  volume~87.
\newblock Springer, 2004.

\bibitem{safran2016quality}
Itay Safran and Ohad Shamir.
\newblock On the quality of the initial basin in overspecified neural networks.
\newblock In {\em International Conference on Machine Learning}, pages
  774--782, 2016.

\bibitem{MLbook}
Shai Shalev-Shwartz and Shai Ben-David.
\newblock {\em Understanding machine learning: From theory to algorithms}.
\newblock Cambridge university press, 2014.

\bibitem{shalev2017failures}
Shai Shalev-Shwartz, Ohad Shamir, and Shaked Shammah.
\newblock Failures of deep learning.
\newblock {\em arXiv preprint arXiv:1703.07950}, 2017.

\bibitem{shamir2016distribution}
Ohad Shamir.
\newblock Distribution-specific hardness of learning neural networks.
\newblock {\em arXiv preprint arXiv:1609.01037}, 2016.

\bibitem{soudry2016no}
Daniel Soudry and Yair Carmon.
\newblock No bad local minima: Data independent training error guarantees for
  multilayer neural networks.
\newblock {\em arXiv preprint arXiv:1605.08361}, 2016.

\bibitem{spiegel1968schaum}
Murray~R Spiegel.
\newblock Schaum’s handbook of formulas and tables.
\newblock {\em MacGraw Hill, New York}, 1968.

\bibitem{williams1997computing}
Christopher~KI Williams.
\newblock Computing with infinite networks.
\newblock {\em Advances in neural information processing systems}, pages
  295--301, 1997.

\bibitem{zhang2016l1}
Yuchen Zhang, Jason~D Lee, and Michael~I Jordan.
\newblock l1-regularized neural networks are improperly learnable in polynomial
  time.
\newblock In {\em International Conference on Machine Learning}, pages
  993--1001, 2016.

\end{thebibliography}
